\documentclass{article}

\usepackage{microtype}
\usepackage{graphicx}
\usepackage{subcaption}
\usepackage{booktabs} 
\usepackage{adjustbox}
\usepackage{multirow}

\usepackage{xcolor}         
\usepackage{amsmath}        
\usepackage{array}
\usepackage{graphicx}
\usepackage{makecell}
\usepackage{bbm}
\usepackage{newfloat}
\usepackage{caption}
\usepackage{tabularx}
\usepackage[ruled,vlined]{algorithm2e}
\SetKwInput{KwInput}{Input}
\SetKwInput{KwOutput}{Output}

\DeclareFloatingEnvironment[]{suppfigure}
\DeclareFloatingEnvironment[]{supptable}

\usepackage{amsmath,amsfonts,bm}

\def\eqref#1{equation~\ref{#1}}

\def\1{\bm{1}}

\def\rvtheta{{\mathbf{\theta}}}
\def\rva{{\mathbf{a}}}
\def\rvb{{\mathbf{b}}}
\def\rvc{{\mathbf{c}}}

\def\rvf{{\mathbf{f}}}

\def\rvr{{\mathbf{r}}}

\def\rvx{{\mathbf{x}}}

\def\rmA{{\mathbf{A}}}

\def\rmW{{\mathbf{W}}}

\def\vc{{\bm{c}}}

\def\ve{{\bm{e}}}

\def\vr{{\bm{r}}}

\DeclareMathAlphabet{\mathsfit}{\encodingdefault}{\sfdefault}{m}{sl}
\SetMathAlphabet{\mathsfit}{bold}{\encodingdefault}{\sfdefault}{bx}{n}

\def\sR{{\mathbb{R}}}

\providecommand{\E}{\mathbb{E}}

\def\cL{{\mathcal{L}}}

\def\cX{{\mathcal{X}}}

\def\Wenc{{\rmW_{\text{enc}}}}
\def\benc{{\rvb_{\text{enc}}}}
\def\Wdec{{\rmW_{\text{dec}}}}
\def\bdec{{\rvb_{\text{dec}}}}

\usepackage{hyperref}

\usepackage[accepted]{icml2026}

\usepackage{amsmath}
\usepackage{amssymb}
\usepackage{mathtools}
\usepackage{amsthm}

\usepackage[capitalize,noabbrev]{cleveref}

\crefname{theorem}{theorem}{theorems}
\crefname{definition}{definition}{definitions}
\crefname{lemma}{lemma}{lemmas}
\crefname{proposition}{proposition}{propositions}
\crefname{algorithm}{algorithm}{algorithms}
\crefname{assumption}{assumption}{assumptions}

\newcommand{\norm}[2]{\left\Vert #2 \right\Vert_{#1}}
\newcommand{\cameraready}[1]{\textcolor{black}{#1}}
\newcommand{\rebuttal}[1]{\textcolor{black}{#1}}

\theoremstyle{plain}
\newtheorem{theorem}{Theorem}[section]
\newtheorem{proposition}[theorem]{Proposition}
\newtheorem{lemma}[theorem]{Lemma}

\theoremstyle{definition}

\newtheorem{assumption}[theorem]{Assumption}
\theoremstyle{remark}
\newtheorem{remark}[theorem]{Remark}

\theoremstyle{plain}

\newtheorem{aproposition}{Proposition}

\usepackage[textsize=tiny]{todonotes}

\icmltitlerunning{Ensembling Sparse Autoencoders}

\begin{document}

\twocolumn[
  \icmltitle{Ensembling Sparse Autoencoders}



  \icmlsetsymbol{equal}{*}

  \begin{icmlauthorlist}
    \icmlauthor{Soham Gadgil}{sch,equal}
      \icmlauthor{Chris Lin}{sch,equal}
      \icmlauthor{Su-In Lee}{sch}
  \end{icmlauthorlist}

  \icmlaffiliation{sch}{Paul G. Allen School of Computer Science \& Engineering, University of Washington}
    \icmlcorrespondingauthor{Soham Gadgil, Chris Lin}{\{sgadgil, clin25, suinlee\}@cs.washington.edu}

  \icmlkeywords{Machine Learning, ICML}

  \vskip 0.3in
]



\printAffiliationsAndNotice{\icmlEqualContribution}

    \begin{abstract}
    Sparse autoencoders (SAEs) are used to decompose neural network activations into human-interpretable features. Typically, features learned by a single SAE are used for downstream applications. However, it has recently been shown that a single SAE captures only a limited subset of features that can be extracted from the activation space. Motivated by this limitation, we introduce and formalize SAE ensembles. Furthermore, we propose to ensemble multiple SAEs through \textit{naive bagging} and \textit{boosting}. In naive bagging, SAEs trained with different weight initializations are ensembled, whereas in boosting SAEs sequentially trained to minimize the residual error are ensembled. Theoretically, naive bagging and boosting are justified as approaches to reduce reconstruction error. Empirically, we evaluate our ensemble approaches with three settings of language models and SAE architectures. Our empirical results demonstrate that, compared to an expanded SAE that matches the number of features in the ensemble, ensembling SAEs improves the reconstruction of language model activations along with SAE stability. Additionally, on downstream tasks such as concept detection and spurious correlation removal, SAE ensembles achieve better performance, showing improved practical utility.
\end{abstract}

\section{Introduction}
\label{sec:intro}
Sparse autoencoders (SAEs) have been shown to decompose neural network activations,  often also described as embeddings or representations, into a high-dimensional and sparse space of human-interpretable features~\citep{cunningham2023sparse, gao2024scaling, lieberum2024gemma, rajamanoharan2024improving}. Recent work has focused on the application of SAEs to language models with interpretability use cases such as detecting concepts~\citep{gao2024scaling, movva2025sparse}, identifying internal mechanisms of model behaviors~\citep{marks2024sparse}, and steering model behaviors~\citep{farrell2024applying, marks2024sparse, o2024steering}. In practice, a single SAE is usually selected for downstream interpretability applications. However, recent work shows that each SAE trained on the same activations captures only a limited subset of the features that can be extracted from the activation space~\citep{fel2025archetypal, paulo2025sparse}. \rebuttal{Therefore, our main question is}: \textit{Can we leverage \rebuttal{multiple} SAEs to improve performance?} 

Leveraging multiple SAEs is well motivated by ensemble methods in supervised learning that utilize model variability to improve predictive performance. Examples \rebuttal{include} bagging (bootstrap aggregating) \rebuttal{that leverages variability due to randomness} ~\citep{breiman1996bagging, breiman2001random} and boosting \rebuttal{that leverages variability due to different optimization objectives}~\citep{chen2016xgboost, friedman2001greedy}. \rebuttal{We} propose to ensemble multiple SAEs and formalize SAE ensembles. Conceptually, SAE ensembles are defined as methods for combining the outputs of SAEs in the activation space. Nonetheless, we show that ensembling the outputs of SAEs corresponds to concatenating the SAE features and feature coefficients. We instantiate two approaches for ensembling SAEs (\Cref{fig:concept}). In \textit{naive bagging}, SAEs differing only in their weight initializations are ensembled. In \textit{boosting}, the ensemble aggregates SAEs that are iteratively trained to reconstruct the residual from previous iterations. In three settings of language models and SAE architectures, our empirical results show that naive bagging and boosting can lead to better reconstruction of language model activations, more interpretable features, and better stability. Finally, to demonstrate the practical utility of our ensemble methods, we apply them to the tasks of concept detection and spurious correlation removal, where ensembling multiple SAEs can achieve better performance than using only one SAE.

\begin{figure*}[t]
\centering
\includegraphics[page=1, width=0.9\textwidth]{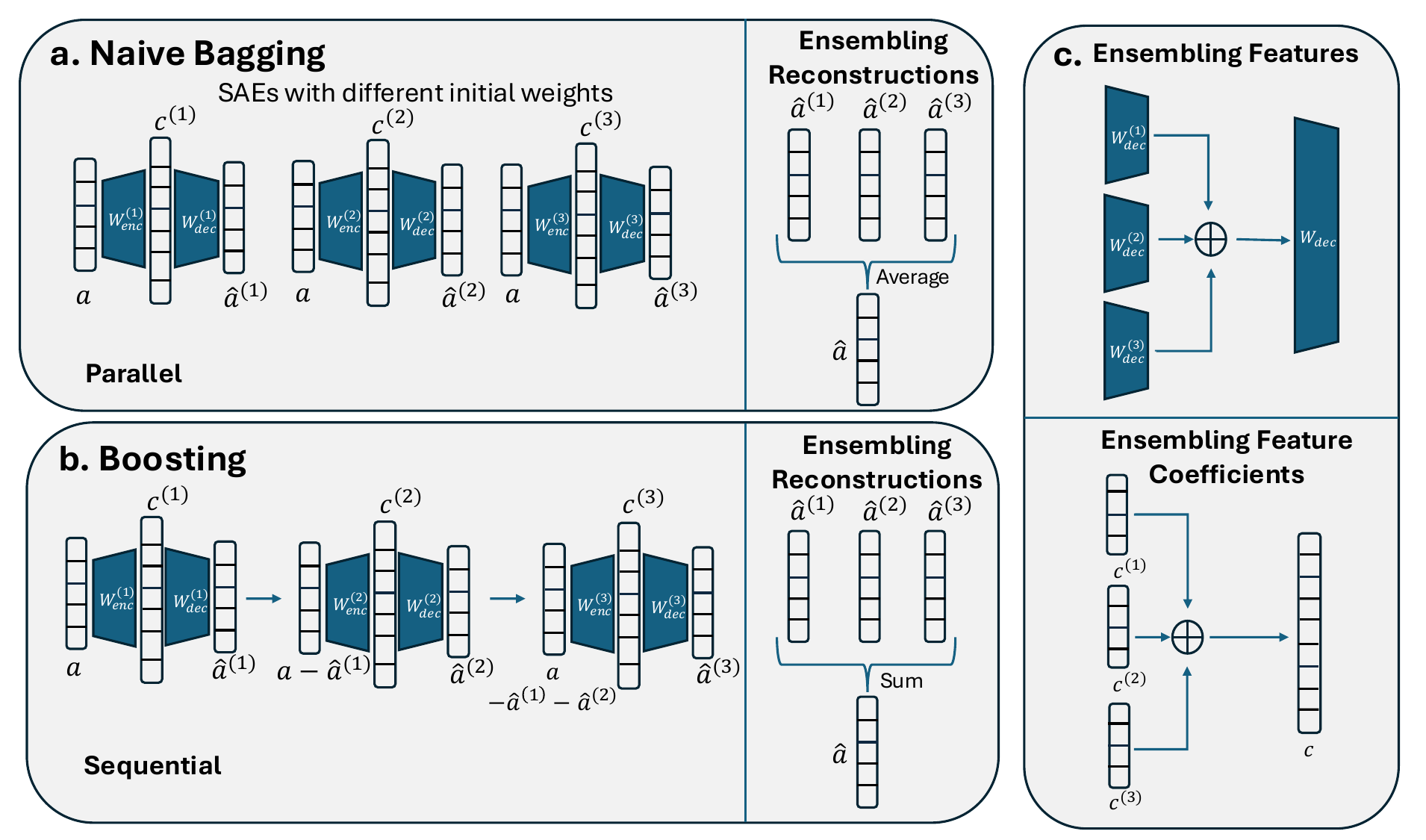}
\caption{Overview of the proposed SAE ensembling strategies. \textbf{a.} \textit{Naive Bagging} involves multiple SAEs with different weight initializations, which can be trained in parallel. The ensembled reconstruction is the average of reconstructions obtained from the individual SAEs. \textbf{b.} \textit{Boosting} involves sequential training of SAEs on the residual error left from the previous iterations. The ensembled reconstruction is the sum of the reconstructions from the individual SAEs. \textbf{c.}  For both approaches, ensembling the features and feature coefficients involves a concatenation.
} \label{fig:concept}
\end{figure*}

\section{Related Work}
\textbf{SAEs.} SAEs have emerged as a scalable and unsupervised approach for extracting human-interpretable features from neural network activations~\citep{fel2023holistic}, with recent work demonstrating their applications to language models~\citep{cunningham2023sparse,gao2024scaling, lieberum2024gemma}. An SAE decomposes neural network activations into sparse linear combinations of features, which are vectors with the same dimensionality as the original activations. Overall, features learned by an SAE can often be annotated with semantic interpretations~\citep{ cunningham2023sparse, rao2024discover}. Because the immediate goal of training an SAE is to decompose activations into sparse combinations of features, intrinsic metrics such as the explained variance of reconstructions and feature connectivity are used to evaluate SAEs~\citep{gao2024scaling, rajamanoharan2024improving, rajamanoharan2024jumping,fel2025archetypal}. At the same time, SAEs are usually trained with the end goal of interpreting language model behaviors, with downstream use cases such as concept detection~\citep{gao2024scaling, movva2025sparse}, mechanistic interpretability~\citep{marks2024sparse}, and model steering~\citep{farrell2024applying, marks2024sparse, o2024steering}. Therefore, metrics specific to downstream applications such as concept detection accuracy and the SHIFT score have been proposed~\citep{karvonen2025saebench}.

\cameraready{Recent work has also proposed new SAE architectures to improve scalability and domain adaptation. For example, Switch SAE trains multiple expert SAEs jointly, and one expert is selected during inference for computational efficiency~\cite{mudide2024efficient}. In contrast, in ensembling multiple SAEs are trained jointly and used jointly during inference. Another related work, SAE Boost, introduces a residual learning approach to correct domain-specific errors by training a secondary SAE to improve domain adaptation~\cite{koriagin2025teach}. However, SAE Boost performs a single residual correction step, unlike general boosting algorithms that could iteratively correct residual errors over multiple sequential rounds.}

\textbf{Variability of SAEs.} In general, the variability \rebuttal{between multiple} SAEs can come from several sources. First, SAEs with different architecture designs can learn different features. For example, it has been shown that the choice of SAE activation function corresponds to assumptions about the separability structure of the features to be learned~\citep{hindupur2025projecting}. The SAE size also has an impact on the types of features learned---a smaller SAE tends to learn high-level features, while a larger SAE tends to learn more specific features~\citep{chanin2024absorption}. Second, given a fixed architecture, SAEs with different training hyperparameters can also learn different features. For example, it has been found that lower learning rates can help reduce the number of dead features that rarely activate~\citep{gao2024scaling}. Finally, it has been shown that SAEs can learn different features even with the same architecture and hyperparameters~\citep{fel2025archetypal, paulo2025sparse}. For the scope of this paper, we focus on the variability and ensembling of SAEs with the same architecture and hyperparameters, consistent with classical ensembling in the supervised setting. Generally, our SAE ensemble approaches can be considered meta-algorithms compatible with any SAE architecture and hyperparameter configuration.

\label{sec:related_work}
\textbf{Model ensembling.} Ensemble methods have been applied to leverage model variability for improving performance, especially in supervised learning. In bagging (bootstrap aggregating), predictions from models trained with bootstrapped data subsets are aggregated~\citep{breiman1996bagging, breiman2001random}. Boosting algorithms train successive models by focusing on the errors made in the previous iterations~\citep{chen2016xgboost, friedman2001greedy}. Stacking is an alternative framework that combines predictions from models with different architectures and inductive biases~\citep{wolpert1992stacked}. More recently, it has been shown that averaging weights of models can lead to improved accuracy without additional inference time~\citep{wortsman2022model, wortsman2022robust}. For unsupervised learning, ensemble methods have mostly been applied to form consensus for clustering and anomaly detection~\citep{aggarwal2013outlier, domeniconi2009weighted, fern2004solving, ghosh2011cluster, zimek2014ensembles}. Motivated by the principle of ensembling, here we propose that SAEs can also be ensembled with respect to their outputs in the activation space. We theoretically show that ensembling SAE reconstructions corresponds to combining SAE features. We also demonstrate that ensembling SAEs can lead to improved intrinsic performance and practical utility when applied to language models.

Our work makes the following contributions. (1) We propose ensembling SAEs as a formal framework, showing that ensembling SAE reconstructions is equivalent to ensembling SAE features. (2) We instantiate two practical ensemble approaches, naive bagging and boosting, with theoretical justifications in relation to reconstruction performance. (3) We empirically demonstrate that ensembling multiple SAEs can improve performance in intrinsic metrics and downstream applications.

\section{Formalizing SAE Ensembles}
This section provides the notation used throughout this paper, the definition of an SAE ensemble, and a theoretical result showing that ensembling SAEs is equivalent to concatenating their features.

\subsection{Notation}\label{sec:notation}
In general, we consider a neural network that maps from a sample space $\cX$ to a $d$-dimensional activation space. An SAE is an autoencoder $g: \sR^d \rightarrow \sR^d$ that reconstructs neural network activations, with the following form:
\begin{align}\label{eq:sae_function_form}
    &g(\rva; \Wenc, \Wdec, \benc, \bdec) \nonumber \\ 
    &= \Wdec h(\Wenc \rva + \benc) + \bdec, 
\end{align}
where $\Wenc \in \sR^{k \times d}, \benc \in \sR^{k}, \Wdec \in \sR^{d \times k}, \bdec \in \sR^d$ are the SAE weights and biases, and $h: \sR^k \rightarrow \sR^k$ is an activation function such as the ReLU, JumpReLU, and TopK functions\rebuttal{~\citep{cunningham2023sparse, lieberum2024gemma, gao2024scaling}.\footnote{\rebuttal{Here, $k$ denotes the SAE dimension and does not denote the number of active features in TopK SAE.}}} Unlike conventional autoencoders, in an SAE we have $k > d$. Notably, the columns of the decoder matrix $\Wdec$ are considered features learned by the SAE. Particularly, let $\Wdec[:, i]$ denote the $i$th column of the decoder matrix. Then $\rvf_i = \Wdec[:, i] \in \sR^d$ is the $i$th feature of the SAE,\footnote{In the literature, the $\rvf_i$'s are associated with different terms such as feature directions and decoder vectors. Here, we follow ~\citet{cunningham2023sparse} and call them features for brevity.} for $i \in [k]$. Furthermore, elements in $\vc = h(\Wenc \rva + \benc) \in \sR^k$ are considered coefficients for the features. Overall, \Cref{eq:sae_function_form} can be rewritten to highlight that an SAE decomposes an activation into features, as follows:
\begin{equation}
    g(\rva; \Wenc, \Wdec, \benc, \bdec) = \sum_{i=1}^k \vc_i \rvf_i + \bdec.
\end{equation}
For conciseness, we let $\rvtheta = (\Wenc, \Wdec, \benc, \bdec)$ denote all the SAE parameters. Finally, we use $\hat{\rva} = g(\rva; \rvtheta)$ to denote the SAE reconstruction.

To train an SAE, a training set of activations $\{\rva^{(n)}\}_{n=1}^N$ are collected by passing a set of samples $\{\rvx^{(n)}\}_{n=1}^N$ through the neural network. Then the SAE parameters are trained to minimize the following empirical loss:
\begin{align}\label{eq:sae_loss_objective}
    &\cL_{\text{SAE}}\left(\{\rva^{(n)}\}_{n=1}^N; \rvtheta\right) = \nonumber \\
    &\frac{1}{N} \sum_{n=1}^N \left[\underbrace{\norm{2}{\rva^{(n)} - g\left(\rva^{(n)}; \rvtheta\right)}^2}_{\text{reconstruction loss}} + \underbrace{\lambda \norm{p}{\rvc^{(n)}}}_{\text{sparsity loss}}\right],
\end{align}
where $\rvc^{(n)} = h(\Wenc \rva^{(n)} + \benc)$ corresponds to the feature coefficients for the $n$th sample, and $\lambda \ge 0$ is the penalty coefficient for the sparsity loss. Note that $\lambda = 0$ for TopK SAEs because sparsity is enforced through the TopK activation function~\citep{gao2024scaling}.

\subsection{SAE Ensembles}\label{sec:ensembling_saes}
In this work we focus on ensembling SAEs with the same architecture. Specifically, given $J$ SAEs with model parameters $\theta^{(j)}$ for $j \in [J]$, an SAE ensemble has the form:
\begin{equation}\label{eq:ensembling_saes}
    \sum_{j=1}^J \alpha^{(j)} g\left(\cdot; \rvtheta^{(j)}\right)
\end{equation}
where $\alpha^{(j)} \ge 0$ is the ensemble weight for the $j$th SAE, and for generality the notation $g(\cdot; \theta^{(j)})$ indicates that each SAE can take arbitrary inputs in $\sR^d$. This weighted-sum formulation is similar to classical ensemble methods, where a weighted sum of outputs from base models is used to make a prediction~\citep{breiman1996bagging, friedman2001greedy}. With an SAE ensemble, the base model is now an SAE.

Different from classical ensembles, ~\Cref{eq:ensembling_saes} by itself does not fully specify an SAE ensemble, since SAE features and their coefficients are also critical components for downstream analyses. Interestingly, because the output of each SAE is a linear combination of its features, ensembling SAEs is equivalent to concatenating their feature coefficients and their decoder matrices (feature vectors). More formally, we have the following proposition, with the proof in Appendix~\ref{app:proofs}.
\begin{proposition}\label{prop:feature_concat}
    Suppose there are $J$ SAEs $g(\cdot; \rvtheta^{(1)}), ..., g(\cdot; \rvtheta^{(J)})$, with decoder matrices $\Wdec^{(1)}, ..., \Wdec^{(J)} \in \sR^{d \times k}$ and decoder biases $\bdec^{(1)}, ..., \bdec^{(J)} \in \sR^{d}$. For a given neural network activation $\rva \in \sR^d$, let $\rvc^{(1)}, ..., \rvc^{(J)} \in \sR^k$ denote the feature coefficients. Then ensembling the $J$ SAEs is equivalent to reconstructing $\rva$ with:
    \begin{equation}
        \hat{\rva} = \Wdec \rvc + \bdec = \sum_{i'=1}^{kJ} \rvc_{i'} \rvf_{i'} + \bdec,
    \end{equation}
    where
    \begin{equation}
        \rvc = 
            \begin{bmatrix}
                \alpha^{(1)} \rvc^{(1)} \\
                \vdots \\
                \alpha^{(J)} \rvc^{(J)}
            \end{bmatrix},
    \end{equation}
    \begin{equation}
        \Wdec =
            \begin{bmatrix}
                \Wdec^{(1)} \cdots \Wdec^{(J)}
            \end{bmatrix},
    \end{equation}
    \begin{equation}
        \bdec = \sum_{j=1}^{J} \alpha^{(j)} \bdec^{(j)},
    \end{equation}
    and $\rvf_{i'} = \Wdec[:, i']$, with $\rvc \in \sR^{kJ}, \Wdec \in \sR^{d \times kJ}, \bdec \in \sR^d$.
\end{proposition}

\begin{remark}
    The ensemble weights $\{\alpha^{(j)}\}_{j=1}^J$ can be folded into either $\rvc$ or \rebuttal{$\Wdec$} for Proposition~\ref{prop:feature_concat} to hold. Since the columns of $\Wdec$ are often constrained to have unit norms to interpret the features as direction vectors~\citep{cunningham2023sparse, rajamanoharan2024improving}, the ensemble weights are folded into $\rvc$ to retain the feature norms.
\end{remark}

\section{Ensemble Methods for SAEs}
\label{sec:methods}
In this section we describe \textit{naive bagging} and \textit{boosting} as two approaches for ensembling SAEs.

\subsection{Naive Bagging}\label{sec:naive_bagging}
Variability of SAEs due to weight initialization is utilized in naive bagging, motivated by prior work showing that SAEs differing only in their initial weights can learn limited but different sets of features~\citep{fel2025archetypal, paulo2025sparse}. Note that we refer to this method as \textit{naive} because, unlike classical bagging, bootstrapped data subsets are not used. This is to ensure that each SAE is trained on the same dataset and isolate the effect of different initializations. Also, as SAEs are often trained on million- or even billion-scale datasets~\citep{gao2024scaling, lieberum2024gemma}, bootstrapping becomes impractical due to memory and storage overhead. Concretely, given $J$ SAEs with different initial weights, naive bagging gives the following ensembled SAE:
\begin{equation}\label{eq:naive_bagging}
    g_{\text{NB}}\left( \rva^{(*)}; \{\rvtheta^{(j)}\}_{j=1}^J \right) = \frac{1}{J} \sum_{j=1}^J g\left( \rva^{(*)}; \rvtheta^{(j)} \right)
\end{equation}
Conceptually, the uniform ensemble weight $\alpha^{(j)} = 1/J$ is motivated by considering naive bagging as a way to reduce reconstruction variance in the bias-variance decomposition (see Proposition \ref{prop:naive_bagging_reconstruction} in Appendix~\ref{app:proofs} for a formal justification).  


\subsection{Boosting}\label{sec:boosting}
Since SAEs with different initial weights still learn some overlapping features~\citep{paulo2025sparse}, naive bagging can result in redundant features in the ensemble. To address this redundancy, we propose a boosting-based ensemble strategy to encourage SAEs to capture different components of a given activation through sequential training. Starting from an initial SAE, each subsequent SAE is trained to capture the residual left from the previous iteration.
Concretely, the $j$th SAE is trained with the following loss:

\begin{align}\label{eq:boosting_loss}
    &\cL_{\text{Boost}}\left( \{\rva^{(n)}\}_{n=1}^N; \rvtheta^{(j)} \right) = \nonumber \\
    &\frac{1}{N} \sum_{n=1}^N \left[ \norm{2}{\rva^{(n, j)} - g\left( \rva^{(n, j)}; \rvtheta^{(j)} \right)}^2 + \lambda \norm{p}{\rvc^{(n, j)}} \right],
\end{align}

where
\begin{equation*}
    \rva^{(n, j)} =
    \begin{cases}
        \rva^{(n)}, & \text{if $j = 1$.} \\
        \rva^{(n)} - \sum_{\ell=1}^{j-1} g\left( \rva^{(n, \ell)}; \theta^{(\ell)} \right), & \text{otherwise.}
    \end{cases}
\end{equation*}
Here, the first iteration corresponds to training an initial SAE with the original activations. For $j > 1$, $\rva^{(n, j)}$ is the residual left from the $(j-1)$th iteration that the $j$th SAE should learn to reconstruct. It is worth noting that the regularization parameters $\lambda$ and $p$ remain the same throughout the training iterations. Intuitively, each SAE in boosting should learn features different from the previous SAEs by capturing the residual. \rebuttal{A recent work adapts the matching pursuit algorithm for SAEs, known as MP-SAEs, which are similar to boosting because residuals are used across multiple iterations~\citep{costa2025flat}. However, MP-SAE is not an ensemble method since the same set of features is shared across all iterations, while each boosting run learns a separate set of features}.

As another motivation, boosting can also lead to good reconstruction performance by bounding the bias term in the bias-variance decomposition (see Proposition~\ref{prop:boosting_reconstruction} in Appendix~\ref{app:proofs} for a formal justification). Overall, given $J$ SAEs trained with \Cref{eq:boosting_loss}, boosting gives the following ensembled SAE:
\begin{equation}\label{eq:boosting}
    g_{\text{Boost}}\left( \rva^{(*)}; \{\rvtheta^{(j)}\}_{j=1}^J \right) = \sum_{j=1}^J g\left( \rva^{(*, j)}; \theta^{(j)} \right).
\end{equation}
\rebuttal{Note that boosting is trained sequentially, and the forward pass also runs sequentially during inference.}



\section{Experiments}
\label{sec:experiments}
In this section, we evaluate our ensemble approaches with intrinsic evaluation metrics (\Cref{intrinsic}), assess the interpretability of the learned features (\Cref{sec:interpretability}), and demonstrate the utility of ensembling SAEs with two use cases: concept detection (\Cref{sec:downstream_concept_detection}) and spurious correlation removal (\Cref{sec:downstream_scr}).

\subsection{Baselines}
As a baseline for each experimental setting, we compare ensemble methods with an expanded SAE trained to have the same number of features as the ensembled SAEs. Since sparsity can have an impact on SAE performance for a given SAE size~\citep{gao2024scaling}, expanded SAEs are trained to have sparsity comparable to the ensembled SAEs, enabling a fair comparison. More details about the expanded SAE baseline are provided in Appendix~\ref{app:expanded_sae}.

\subsection{Evaluating Ensembled SAEs with Intrinsic Metrics}
\label{intrinsic}

\subsubsection{Setup}
\label{sec:setup}
We evaluate our ensemble approaches on SAEs trained with activations from three different language models: GELU-1L, Pythia-160M, and Gemma 2-2B, which represent a range of model sizes. Following prior work, ReLU, TopK, and JumpReLU SAEs are trained with the residual stream activations from GELU-1L~\citep{bricken2023monosemanticity}, layer 8 from Pythia-160M~\citep{gao2024scaling}, and layer 12 from Gemma 2-2B~\citep{lieberum2024gemma}, respectively.
Per-token activations are obtained from the Pile~\citep{gao2020pile} for each language model with the corresponding context size. For training the SAEs, we use 800 million tokens from a version of the Pile with copyrighted contents removed.\footnote{\url{https://huggingface.co/datasets/monology/pile-uncopyrighted}} A held-out test set of 7 million tokens is used for evaluation. Hyperparameters are swept for the base SAE, and hyperparameters giving an explained variance closest to 90\% are selected. This ensures that the SAEs being ensembled are practically usable to explain the activations. All SAEs are trained using the Adam optimizer~\citep{kingma2014adam}.\footnote{The code to implement and evaluate the ensembling methods has been submitted as part of the
supplementary material.} Additional details about the language models along with training times, \rebuttal{inference times,} and hyperparameter selection are provided in Appendix~\ref{app:implementation_details}. 


\subsubsection{Metrics}
\label{sec:intrinsic_metrics}
We evaluate different aspects of the ensembled SAEs using four intrinsic metrics: Explained Variance, Mean Squared Error (MSE), Connectivity, and Stability:

\textbf{Reconstruction performance.} We use two standard metrics, mean squared error (MSE) and explained variance, to evaluate the reconstruction of activations:
    \begin{equation*}
        \text{MSE} =\frac{1}{N}\sum_{n=1}^N \norm{2}{\rva^{(n)}-\hat{\rva}^{(n)}}^2 \text{, and}
    \end{equation*}
    \begin{equation*}
        \text{Explained Variance} =\frac{1}{d}\sum_{q=1}^d \left[
            1 
            - \frac{
                \sum_{n=1}^N (\rva^{(n)}_q -\hat{\rva}^{(n)}_q)^2
            }{
                \sum_{n=1}^{N} (\rva^{(n)}_q-\bar{\rva}_q)^2
            }
        \right],
    \end{equation*}
where $d$ is the activation dimensionality, and $\bar{\rva}_q$ is the mean activation for the $q$th dimension.





\textbf{Connectivity.} This metric, proposed in~\citet{fel2025archetypal}, measures the number of distinct pairs of SAE feature coefficients that are activated together across samples. It quantifies the diversity of the feature coefficients: 

\begin{equation*}
    \text{Connectivity}=1-\left(\frac{1}{m^2}\norm{0}{\mathbf{C}^\top \mathbf{C}}\right), 
\end{equation*}
where $\mathbf{C}\in\sR^{N\times m}$ is the matrix of feature coefficients across all samples, and here $\norm{0}{\cdot}$ counts the number of non-zero elements in a matrix.

\textbf{Stability.} This metric, adapted from~\citet{paulo2025sparse}, measures the maximum cosine similarity of the features that can be obtained across multiple runs of SAE training (with or without ensembling). Higher stability corresponds to the discovery of features that are similar across different runs. Note that this metric does not depend on the evaluation tokens. Given a total of $S$ training runs, the stability for the $s$th run is:
\begin{equation*}
    \text{Stability}= \frac{1}{m}\sum_{i=1}^m\max_{s'\in\left[S\right] \setminus s, j \in [m]}\langle {\rvf_i^{(s)}},\rvf_j^{(s')}\rangle.
\end{equation*}

Stability is computed across 5 runs and the confidence intervals are obtained by using each of the runs as the base run for computing stability with the others.


\subsubsection{Results} 
Figure \ref{fig:gemma_eval} illustrates how the number of SAEs in the ensemble affects intrinsic performance for both naive bagging and boosting on Gemma 2-2B. The first point in each plot represents the base SAE. Increasing the number of SAEs in the ensemble generally improves performance for most metrics and maintains the performance for the others. Comparing the two ensemble approaches, boosting outperforms naive bagging across all metrics except for stability. This is consistent with the theoretical justification that naive bagging reduces variance (\Cref{sec:naive_bagging}). On the other hand, since boosting aims for bias reduction (\Cref{sec:boosting}), it can learn more specific and low-level features, impacting stability. Results for GELU-1L and Pythia-160M are provided in Appendix~\ref{app:add_intrinsic_eval}, where similar trends hold. 

Detailed results for ensembles of 8 SAEs across all three language models are summarized in Table \ref{tab:ensemble_eval}. We ensemble with 8 SAEs as most of the metrics begin to plateau by then.  Compared to an expanded SAE, naive bagging (NB) performs better in stability while worse in the other intrinsic metrics such as the reconstruction metrics for GELU-1L and Pythia-160M. This is expected due to the stability-reconstruction tradeoff~\citep{fel2025archetypal}. However, naive bagging can also be applied to the expanded SAE as a way to gain both reconstruction performance and stability. Notably, the stability of the expanded SAE is typically less than half of the stability of ensembled SAEs, indicating that a larger SAE can result in unreliable features. More importantly, boosting outperforms an expanded SAE in the reconstruction metrics and stability, while having similar connectivity scores. This comparison highlights that the gains from ensembling are not just because the ensembled SAEs have more features. This comparison also shows that boosting is a strong alternative to expanding SAE size, especially for its better stability in applications that require interpretability tools to be reliable~\citep{fel2025archetypal, paulo2025sparse}. Overall, we find that ensembling performs better than an expanded SAE on the intrinsic metrics. As a sanity check, ensembling with 8 SAEs also performs better in all the intrinsic metrics compared to the base SAE (Supplementary Table \ref{stab:ensemble_eval}). 

\setlength{\tabcolsep}{1pt}
\begin{table}[h!]
\centering
\small
\caption{Intrinsic evaluation metrics for an expanded SAE, naive bagging (NB), and boosting (ensembling 8 SAEs). Means along with 95\% confidence intervals are reported across 5 runs.}
\vspace{0.3em}
\resizebox{0.49\textwidth}{!}{
\begin{tabular}{p{3cm}p{2cm}p{2.3cm}p{2.2cm}l}

\toprule
\textbf{Ensembling Method} & \textbf{Explained \hspace{2em} Variance ($\uparrow$)}            & \textbf{MSE ($\downarrow$)}                & \textbf{Connectivity ($\uparrow$)}             & \textbf{Stability ($\uparrow$)}      \\
\midrule
\multicolumn{5}{l}{\textbf{GELU-1L}}                 \\
\midrule
Expanded SAE      & 0.946 (0.0003) 	& 17.893 (0.137)	&	\textbf{0.959 (0.0003)} &	0.372 (0.0022)          \\
Ensembling (NB)     & 0.895 (0.0006)                   & 35.147 (0.210)                         & 0.307 (0.0009)           & \textbf{0.745 (0.0002)} \\
Ensembling (Boosting)          & \textbf{0.961 (0.0018)}  & \textbf{12.542 (0.589)}    & 0.945 (0.0004)  & 0.707 (0.0014)          \\
\midrule
\multicolumn{5}{l}{\textbf{Pythia-160M}}                                                                                                                                     \\
\midrule
Expanded SAE       & 0.987 (0.0041)	& 4.387 (1.486) 	& 0.978 (0.0006) &	0.204 (0.0006)         \\
Ensembling (NB)    & 0.929 (0.0000)           & 24.704 (0.019)             & 0.912 (0.0006)           & \textbf{0.731 (0.0017)} \\
Ensembling (Boosting)          & \textbf{0.998 (0.0021)} & \textbf{0.845 (0.547)}   & \textbf{0.986 (0.0004)}  & 0.680 (0.0025)          \\
\midrule
\multicolumn{5}{l}{\textbf{Gemma 2-2B}}                                                                                                                                      \\
\midrule
Expanded SAE      & 0.948 (0.0012) &	472.330 (10.759)	& \textbf{0.993 (0.0003)} & 0.268 (0.0021)          \\
Ensembling (NB)     & 0.974 (0.0006) & 234.128 (6.228) & 0.769 (0.0007)  & \textbf{0.633 (0.0014)} \\
Ensembling (Boosting)          & \textbf{0.995 (0.0003)} & \textbf{46.538 (2.923)} & 0.989 (0.0003) & 0.583 (0.0009)   \\  
\bottomrule
\end{tabular}
}

\label{tab:ensemble_eval}
\end{table}

\begin{figure}[h!]
\centering
\includegraphics[width=0.49\textwidth]{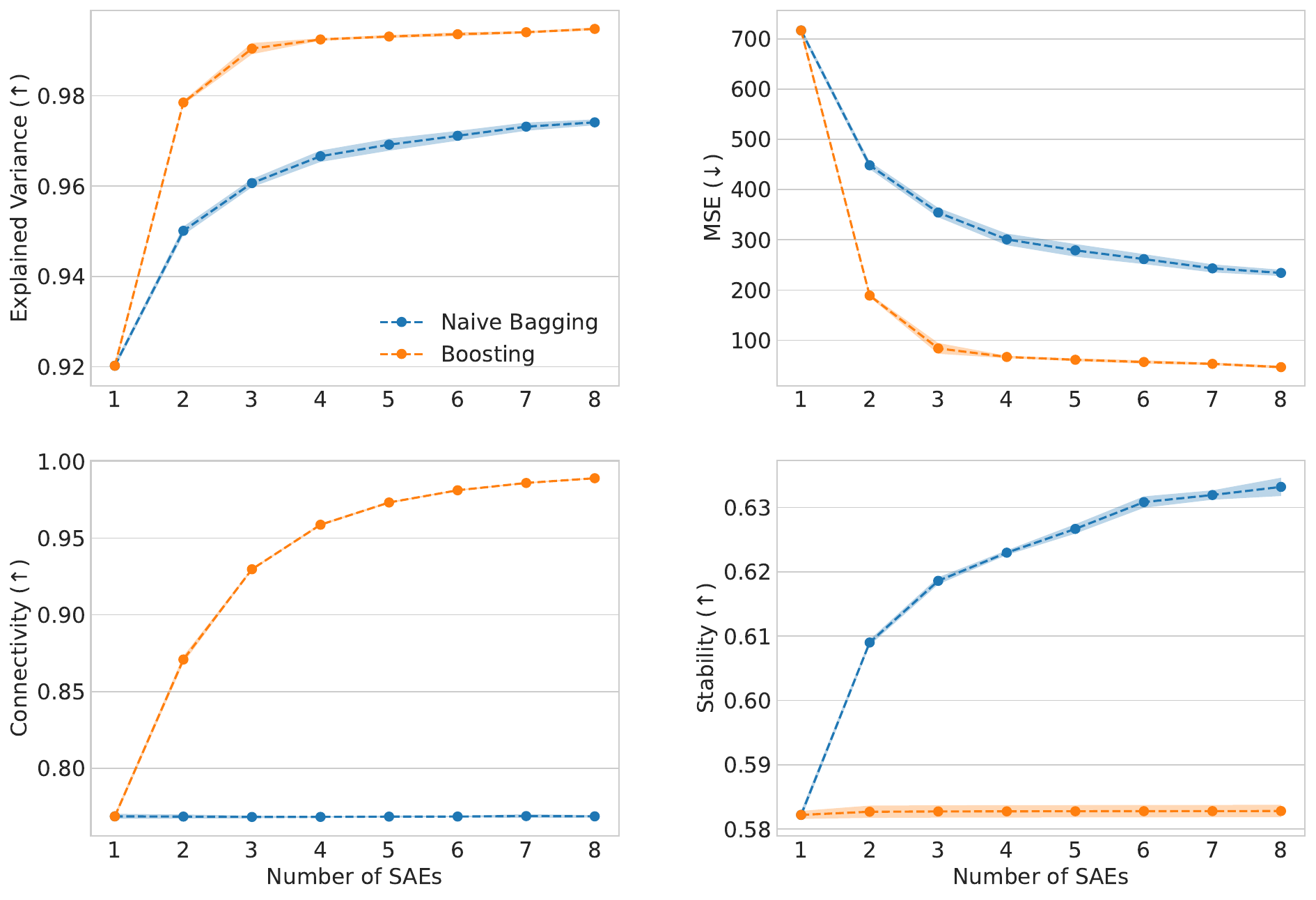}
\caption{\rebuttal{Ablation study to investigate the effect} of the number of SAEs in the ensemble for naive bagging and boosting on the intrinsic evaluation metrics for Gemma 2-2B (Layer 12). The shaded regions indicate 95\% confidence intervals across 5 different experiment runs. For naive bagging, the different experiment runs correspond to different sets of initial weights.}
\label{fig:gemma_eval}
\end{figure}

\subsection{Interpretability of Features Learned by Ensembled SAEs}
\label{sec:interpretability}
The sparsity of an SAE's feature coefficients, measured by the $L_0$ norm, is often considered a proxy for the interpretability of the SAE's features. In an SAE ensemble, each constituent SAE is trained with the same sparsity constraint and therefore remains individually sparse. However, a potential concern is that the ensemble as a whole has a large aggregate $L_0$ because feature coefficients from all constituent SAEs are concatenated, which could be considered as evidence that the learned features are not interpretable. To directly assess this concern, we evaluate the interpretability of features learned by SAE ensembles using language models through automated interpretability (AutoInterp)~\citep{paulo2410automatically}. As shown in \Cref{tab:autointerp}, the ensembled SAEs achieve higher average AutoInterp scores than the expanded SAEs. Moreover, the ensembled SAEs have similar or better AutoInterp scores than the corresponding base SAEs, despite having larger aggregate $L_0$ norms (Supplementary Table \ref{stab:autointerp}). Taken together, these results suggest that aggregate $L_0$ is not a reliable proxy for interpretability in the context of SAE ensembles, adding to growing evidence in the SAE literature that sparsity alone does not fully capture feature interpretability~\citep{sharkey2025open}.

\cameraready{We also use AutoInterp to perform a qualitative analysis of the features learned by the first and last SAEs for both the ensembling techniques (\Cref{tab:sae_concepts}). For naive bagging, we observe that the features learned by both the first and the last SAEs are at a similar hierarchical level (applicable to a broad category of examples). In the case of Pythia-70M, the features are semantically similar as well. For boosting, the first and the last SAEs tend to learn features at different conceptual levels. The first SAE captures more broad, high-level features, while the last SAE captures specific features applicable to a smaller set of samples. For example, in the case of Pythia-160M, the first SAE learns a feature for words associated with programming concepts while the last SAE learns specific terms like `tokens', `syntax error', and `Java'. These results suggest that boosting can improve reconstruction performance by focusing on more specific features in the later SAEs.}

\begin{table*}[t]
\centering
\caption{Comparison of concepts identified by the First SAE and Last SAE across methods and models.}
\begin{tabularx}{\textwidth}{p{2.4cm} @{\hspace{10pt}} X @{\hspace{10pt}} X}
\toprule
\textbf{Model} & \textbf{First SAE} & \textbf{Last SAE} \\
\midrule

\multicolumn{3}{l}{\textbf{Ensembling (NB)}} \\
\midrule
\addlinespace[2pt]

GELU-1L 
& concepts related to personal conflict, emotional turmoil, and family dynamics 
& terms related to development, planning, management, and constructions \\

Pythia-70M 
& legal terms related to defense and prosecution in court contexts 
& phrases related to legal arguments and courtroom procedures \\

Pythia-160M 
& words related to specific concepts or formal elements in academic writing 
& questions expressing confusion or curiosity about reasons or explanations \\

Gemma-2-2B 
& words related to contacting or communication 
& expressions of surprise, awe, and disbelief \\

\midrule

\multicolumn{3}{l}{\textbf{Ensembling (Boosting)}} \\
\midrule
\addlinespace[2pt]

GELU-1L 
& the word `where' indicating locations or contexts in various situations 
& technical terms and symbols in equations and scientific discussions \\

Pythia-70M 
& the word `your' and related second-person pronouns indicating personal involvement or responsibility 
& statistical significance and associated numerical representation in research and analysis \\

Pythia-160M 
& words and substrings associated with programming concepts and technical terms 
& text related to technology and programming, including specific terms like ``tokens,'' ``syntax error,'' and ``Java'' \\

Gemma-2-2B 
& technical terms related to software settings and programming elements 
& specific numerical references and words related to documentation, permissions, and experimental contexts \\

\bottomrule
\end{tabularx}
\label{tab:sae_concepts}
\end{table*}

\begin{table}[h!]
\centering
\caption{AutoInterp scores for an expanded SAE, naive bagging (NB), and boosting (ensembling 8 SAEs). Means along with 95\% confidence intervals are reported across 5 experiment runs.}
\resizebox{0.49\textwidth}{!}{
\begin{tabular}{p{3.3cm}>{\centering}p{2cm}>{\centering}p{2cm}>{\centering\arraybackslash}p{2cm}>{\centering\arraybackslash}p{2cm}} \\
\toprule
\textbf{}      & \textbf{GELU-1L} & \textbf{Pythia-160M} & \textbf{Gemma 2-2B} \\
\midrule
Expanded SAE   &	0.714 (0.051) &	0.852 (0.006) &	0.805 (0.002) \\
Ensembling (NB)                                    & 0.738 (0.147)                     & \textbf{0.857(0.004)}  &   0.799 (0.008)           \\
Ensembling (Boosting)        & \textbf{0.863 (0.005)}                & 0.852 (0.018) & \textbf{0.814 (0.002)} \\
\bottomrule
\end{tabular}
}
\label{tab:autointerp}
\end{table}

\subsection{Use Case 1: Concept Detection}
\label{sec:downstream_concept_detection}
Interpretability use cases of SAEs such as debiasing, understanding sparse circuits, and hypothesis generation often require individual SAE features to correspond to semantic concepts~\citep{cunningham2023sparse, marks2024sparse, movva2025sparse}. Therefore, here we apply our ensemble approaches to detect semantic concepts across a range of domains. Specifically, per-token activations are encoded using an ensembled SAE, and mean-pooling is applied to obtain a sequence-level embedding. The SAE feature having the maximum mean difference between samples with and without the concept in the training set is selected to train a logistic regression classifier. Finally, accuracy on a held-out test set is used to evaluate the concept detection performance. We note that this evaluation procedure follows prior work~\citep{gao2024scaling, karvonen2025saebench}.

 \textbf{Setup.} We train a ReLU SAE as the base SAE on the residual stream activations from layer 4 of Pythia-70M, with 100 million tokens from the Pile~\citep{gao2020pile}. This setting is chosen since it has been used for concept-level tasks~\citep{karvonen2024evaluating, marks2024sparse}. Our concept detection use case encompasses four datasets: \textbf{(1) Amazon Review (Sentiment):} classifying the sentiment of the review (1 vs. 5 stars), \textbf{(2) GitHub Code:} identifying the coding language from source code, \textbf{(3) AG News:} classifying news articles by topics, and \textbf{(4) European Parliament:} detecting the language of a document.

\textbf{Results.} Table \ref{tab:concept_detection} illustrates the results of the concept detection task for our ensemble approaches (with 8 SAEs in each ensemble). Comparing the two ensemble approaches, naive bagging generally performs better than boosting \rebuttal{in terms of the mean performance}. One \rebuttal{possible} reason for the higher performance of naive bagging could be that it identifies features at a conceptual hierarchy which is suitable for this task, while boosting can potentially identify \rebuttal{low-level} features \rebuttal{useful for the training set but do not generalize to the test set}. However, multiple specific features can be combined to detect a more general concept. Indeed, boosting can perform better than naive bagging when the top 5 concept-associated features are considered instead of using only the top feature (Supplementary Table~\ref{stab:concept_detection_top_5}). Therefore, naive bagging should be used for applications where each concept is mapped to only one SAE feature, whereas boosting excels when each concept is mapped to multiple SAE features. \cameraready{In particular, boosting outperforms naive bagging starting when the top 2 concept-associated features are considered for concept detection (\Cref{stab:concept_detection_top_2}).} \rebuttal{Overall, considering the means along with the confidence intervals, we observe that ensembling performs slightly better than an expanded SAE across all the concept detection tasks (\Cref{tab:concept_detection})}. As a sanity check, ensembling also performs better than the base SAE (Supplementary Table \ref{stab:concept_detection}). \rebuttal{For completeness, evaluation of SAEs trained for the other language models on the concept detection tasks is shown in Appendix \ref{app:downstream}}.

\begin{table}[h!]
\centering
\caption{Test accuracy of the logistic regression classifier for the top concept-associated feature across four concept detection tasks for ensembles with 8 SAEs. Means along with 95\% confidence intervals are reported across 5 experiment runs.}
\vspace{0.3em}

\resizebox{\columnwidth}{!}{
\begin{tabular}{p{3.3cm}>{\centering}p{2cm}>{\centering}p{2cm}>{\centering\arraybackslash}p{2cm}>{\centering\arraybackslash}p{2cm}}
\toprule
\textbf{}      & \textbf{Amazon Review \hspace{2em} (Sentiment)} & \textbf{GitHub Code \hspace{4em} (Language)} & \textbf{AG News \hspace{2em}(Topic)} & \textbf{European Parliament \hspace{2em}(Language)} \\
\midrule
Expanded SAE   & 	0.600 (0.032) &	0.682 (0.025) &	0.746 (0.021) &	0.942 (0.009) \\
Ensembling (NB)                 & \textbf{0.631 (0.036)}                     & \textbf{0.715 (0.012)}                     & 0.742 (0.037)   &   \textbf{0.943 (0.016)}           \\
Ensembling (Boosting)        & 0.624 (0.037)            & 0.682 (0.021)            & \textbf{0.759 (0.021)} & 0.920 (0.015) \\
\bottomrule
\end{tabular}
}
\label{tab:concept_detection}
\end{table}

\setlength{\tabcolsep}{5pt}
\begin{table}[h!]
\centering
\caption{Test accuracy of the logistic regression classifier for the top-2 concept-associated feature across four concept detection tasks for Pythia-70M. SAE Ensembles consist of 8 SAEs. Means along with 95\% confidence intervals are reported across 5 experiment runs.}   
\resizebox{\columnwidth}{!}{
\begin{tabular}{p{3.3cm}>{\centering}p{2cm}>{\centering}p{2cm}>{\centering\arraybackslash}p{2cm}>{\centering\arraybackslash}p{2cm}}
\toprule
\textbf{}    & \textbf{Amazon Review \hspace{2em} (Sentiment)} & \textbf{GitHub Code \hspace{4em} (Language)} & \textbf{AG News \hspace{2em}(Topic)} & \textbf{European Parliament \hspace{2em}(Language)} \\
\midrule

Ensembling (NB) & 0.655 (0.035)            & 0.718 (0.018)                    & 0.765 (0.041)   &   0.943 (0.011)          \\
Ensembling (Boosting)   & \textbf{0.658 (0.024)}            & \textbf{0.725 (0.016)}            & \textbf{0.803 (0.019)} & \textbf{0.963 (0.010)} \\
\bottomrule
\end{tabular}
}

\label{stab:concept_detection_top_2}
\end{table}

\subsection{Use Case 2: Spurious Correlation Removal}
\label{sec:downstream_scr}
Neural networks have been previously shown to encode spurious correlations between non-essential input signals (e.g. image background) and the target label, which can negatively impact their generalization performance, robustness, and trustworthiness~\citep{degrave2021ai, ye2024spurious}. Such biases can get exacerbated in more complex networks like large language models~\citep{kotek2023gender, navigli2023biases}. Motivated by this, we consider the task of spurious correlation removal (SCR), as proposed in~\citet{karvonen2024evaluating}. The evaluation procedure here follows~\citet{karvonen2025saebench} and is an automated version of Sparse Human-Interpretable Feature Trimming (SHIFT) by~\citet{marks2024sparse}.

\textbf{Setup.} The goal of SCR is to identify specific SAE features for the spurious signal and debias a classifier by ablating those features. Here we use the Bias in Bios dataset~\citep{de2019bias}, which maps professional biographies to profession and gender. First, the dataset is filtered for a pair of professions (e.g. professor and nurse) and then it is partitioned into two sets: one which is balanced in terms of profession and gender, and the other with biased gender association for a particular profession (e.g. male professors and female nurses). Then, a linear classifier $C_b$ is trained on the biased set using the activations from a language model. The goal is to debias this classifier using the features identified by the SAE to improve the accuracy on classifying profession in an unbiased held-out set. 

For debiasing the classifier $C_b$, a set of top $L$ SAE features is first identified based on their probe attribution scores for a probe trained to predict the spurious signal (i.e. gender)~\citep{karvonen2025saebench}. We use the same base SAE setup as the one used in the concept detection task -- a ReLU SAE trained using Pythia-70M activations with 100 million tokens. Then, a modified classifier $C_m$ is trained after removing the spurious signal by zero-ablating the $L$ SAE features. The predictive performance of the modified classifier $C_m$ on profession for the held-out, balanced dataset indicates the SAE quality. Following~\citet{karvonen2025saebench}, the normalized evaluation score $S_{\text{SHIFT}}$ is defined as: 

\begin{equation*}
    S_{\text{SHIFT}}=\frac{A_{\text{abl}}-A_{\text{base}}}{A_{\text{oracle}}-A_{\text{base}}
    },
\end{equation*}
where $A_{\text{abl}}$ is the accuracy for $C_m$, $A_{\text{base}}$ is the accuracy for $C_b$, and $A_{\text{oracle}}$ is the oracle accuracy with a classifier trained on a balanced dataset. It is worth noting that $A_{\text{base}}$ and $A_{\text{oracle}}$ do not depend on SAEs.


\textbf{Results.} The (ensembled) 
SAEs from \Cref{sec:downstream_concept_detection} are used here, with $L = 20$ features selected, following prior work~\citep{karvonen2025saebench}. Table \ref{tab:scr} shows the performance of our ensemble approaches for the SCR task across four pairs of profession, with the first profession biased towards males and the second towards females. Comparing the ensemble approaches, naive bagging does not perform as well as the baselines, which could be because in naive bagging there are more than $L$ similar features related to the spurious signal, and all of those features need to be ablated to observe an improved $A_{\text{abl}}$. In contrast, boosting outperforms naive bagging and the baselines, suggesting that it is more effective in isolating and removing gender-related features. Overall, these results show that ensembling can outperform the expanded SAE across all pairs of professions. As a sanity check, boosting also performs better than the base SAE across all profession pairs, while naive bagging does not perform as well (Supplementary Table \ref{stab:scr}). Similar trends are observed as the number of top gender-related features $L$ is further increased (Supplementary Figure~\ref{sfig:scr_vs_topk_ablation}). \rebuttal{For completeness, evaluation of additional SAEs on the SCR task is shown in Appendix \ref{app:downstream}.}

\begin{table}[h!]
\centering
\caption{$S_\text{SHIFT}$ scores for the spurious correlation removal task with the top 20 gender-related features identified across four pairs of professions for ensembles with 8 SAEs. Means along with 95\% confidence intervals are reported across 5 experiment runs.}
\vspace{0.3em}
\resizebox{\columnwidth}{!}{
\begin{tabular}{l>{\centering}p{2cm}>{\centering}p{2cm}>{\centering}p{2.6cm}>{\centering\arraybackslash}p{2cm}}
\toprule
\textbf{}     & \textbf{Professor \hspace{2em} vs. Nurse} & \textbf{Architect \hspace{2em} vs. Journalist} & \textbf{Surgeon \hspace{4em} vs. Psychologist} & \textbf{Attorney \hspace{2em} vs. Teacher} \\
\midrule
Expanded SAE  & 0.047 (0.014) &	0.006 (0.005) &	0.037 (0.009) &	0.021 (0.007) \\ 
Ensembling (NB) & 0.021 (0.003)                & 0.004 (0.001)                     & 0.014 (0.002)                     & 0.003 (0.005)                 \\
Ensembling (Boosting)      & \textbf{0.066 (0.016)}       & \textbf{0.013 (0.011)}            & \textbf{0.045 (0.014)}            & \textbf{0.029 (0.003)} \\
\bottomrule
\end{tabular}
}
\label{tab:scr}
\end{table}

\subsection{Evaluation with additional downstream metrics}
The use cases in \Cref{sec:downstream_concept_detection} and \Cref{sec:downstream_scr} are adapted from SAEBench~\cite{karvonen2025saebench}, which implements additional metrics to evaluate SAEs. We perform an extensive evaluation on the other metrics from SAEBench, including the RAVEL  (Resolving Attribute-Value Entanglements in Language Models) score, the CE (Cross Entropy) loss score, and TPP (Targeted Probe Perturbation). For all the metrics, we use the default hyperparameter values provided in the code repository.\footnote{\url{https://github.com/adamkarvonen/SAEBench/tree/main}} We observe that our ensembling approaches are able to outperform or perform similarly to the baselines (Appendix \ref{app:saebench}) even on additional downstream metrics. This highlights that the performance improvements gained by ensembling are not limited to the use cases of concept detection and spurious correlation removal, but can be generalized to multiple other downstream tasks.

\section{Discussion}
\label{sec:discussion}
In this work, we propose and formalize ensembling SAEs as a way to improve performance by leveraging the feature variability of SAEs with the same architecture and hyperparameters. We instantiate two ensembling approaches, \textit{naive bagging} and \textit{boosting}. Theoretically, we justify both approaches as ways to improve reconstruction and
show that ensembling in the output space of SAEs is equivalent to concatenation in the feature space. Empirically, we show that ensembling can improve intrinsic performance, leading to better reconstruction of language model activations, more diverse feature coefficients, and improved stability. We also demonstrate the practical utility of our ensembling approaches through quantitative validation on two downstream use cases, where ensembling can also lead to performance improvement. 

Our ensemble approaches do come with some limitations. Both naive bagging and boosting are computationally more expensive than training the base SAE, since they require multiple SAEs to be trained. While this can be run in parallel for naive bagging, boosting has to be run sequentially. While ensembling performs better than the base SAE across all intrinsic metrics, this does not always translate to better downstream performance. For example, naive bagging could result in redundant features, causing a performance drop for tasks where multiple features are selected for ablation. On the other hand, boosting could learn features that are too specific, leading to lower performance for detecting high-level concepts with individual features. Thus, different ensemble approaches should be used based on the specific goals and procedures of downstream applications. \rebuttal{Furthermore, similar to ensemble methods in supervised learning, ensemble approaches for SAEs can incur higher model complexity compared to the base SAE. However, individual SAEs in an ensemble remain sparse, and the learned features are interpretable based on downstream evaluations.}  

As a framework, SAE ensembling can be considered a meta-algorithm, which can be extended to different settings. We scope this work to focus on SAEs with the same architecture\rebuttal{, hyperparameters, and training data}, but future directions can consider ensembling (stacking) different \rebuttal{SAE} architectures \rebuttal{such as Matryoshka SAEs~\citep{bussmann2025learning} and Switch SAEs~\citep{mudide2024efficient}, or SAEs trained on data from different subdomains~\citep{muhamed2025decoding}}. Beyond language models, ensembling can also be used for SAEs trained on activations from models of other input domains (e.g. activations from vision models). Finally, future work can also explore ensembling from theoretical perspectives beyond reconstruction, such as feature identification.

\clearpage
\section*{Impact Statement}
This paper proposes ensembling SAEs to leverage their feature variability. By improving reconstruction fidelity, feature stability, and downstream utility, our work contributes to the development of more reliable and robust interpretability tools for language models. 

Improved interpretability methods can enable better auditing of model behaviors, identification of harmful biases, and responsible deployment of these models in real-world settings. In particular, our experiments on spurious correlation removal demonstrate potential for mitigating gender bias in downstream classifiers. However, improved interpretability techniques can also be misused by bad actors. More effective feature identification can enable adversaries to introduce undesirable model behaviors through targeted steering. Additionally, careful validation of the identified features is needed to ensure accurate interpretation of model behaviors.

\section*{Acknowledgment}
We thank members of the Lee lab for providing feedback
on this project and the reviewers for their constructive com-
ments. This work was funded by the National Institutes of Health [R01 AG061132, R01 EB035934, RF1 AG088824].

\bibliography{reference}
\bibliographystyle{icml2026}

\newpage
\appendix
\onecolumn
\section*{Appendix}
\section{Theoretical Results}\label{app:proofs}
Here we (re-)state and prove our results from \Cref{sec:ensembling_saes}, \Cref{sec:naive_bagging}, and \Cref{sec:boosting}.

\begin{aproposition}
    Suppose there are $J$ SAEs $g(\cdot; \rvtheta^{(1)}), ..., g(\cdot; \rvtheta^{(J)})$, with decoder matrices $\Wdec^{(1)}, ..., \Wdec^{(J)} \in \sR^{d \times k}$ and decoder biases $\bdec^{(1)}, ..., \bdec^{(J)} \in \sR^{d}$. For a given neural network activation $\rva \in \sR^d$, let $\rvc^{(1)}, ..., \rvc^{(J)} \in \sR^k$ denote the feature coefficients. Then ensembling the $J$ SAEs is equivalent to reconstructing $\rva$ with:
    \begin{equation}
        \hat{\rva} = \Wdec \rvc + \bdec = \sum_{i'=1}^{kJ} \rvc_{i'} \rvf_{i'} + \bdec,
    \end{equation}
    where
    \begin{equation}
        \rvc = 
            \begin{bmatrix}
                \alpha^{(1)} \rvc^{(1)} \\
                \vdots \\
                \alpha^{(J)} \rvc^{(J)}
            \end{bmatrix}
        \text{, }
        \Wdec =
            \begin{bmatrix}
                \Wdec^{(1)} \cdots \Wdec^{(J)}
            \end{bmatrix}
        \text{, }
        \bdec = \sum_{j=1}^{J} \alpha^{(j)} \bdec^{(j)}, 
    \end{equation}
    and $\rvf_{i'} = \Wdec[:, i']$, with $\rvc \in \sR^{kJ}, \Wdec \in \sR^{d \times kJ}, \bdec \in \sR^d$.
\end{aproposition}

\begin{proof}
    Based on the definition of an SAE ensemble in \Cref{eq:ensembling_saes} and the definition of feature coefficients, we have
    \begin{align}
        \hat{\rva}
            &= \sum_{j=1}^J \alpha^{(j)} \left(\Wdec^{(j)} \rvc^{(j)} + \bdec^{(j)} \right) \label{eq:prop1_step1} \\
            &=
                \begin{bmatrix}
                    \Wdec^{(1)} \cdots \Wdec^{(J)}
                \end{bmatrix}
                \begin{bmatrix}
                    \alpha^{(1)} \rvc^{(1)} \\
                    \vdots \\
                    \alpha^{(J)} \rvc^{(J)}
                \end{bmatrix}
                + \sum_{j=1}^J \alpha^{(j)} \bdec^{(j)}\label{eq:prop1_step2} \\
            &= \Wdec \rvc + \bdec\label{eq:prop1_step3},
    \end{align}
    where \Cref{eq:prop1_step2} follows from observing that the sum of matrix-vector product is equivalent to the product of the concatenated matrix and vector.
\end{proof}

Here, we provide a lemma showing the bias-variance decomposition for reconstructing a neural network activation with an ensembled SAE (\Cref{sec:ensembling_saes}).

\begin{lemma} \label{lem:ensemble_bias_variance_decomp}
    Given a neural network activation $\rva^{(*)}$, and the ensembled SAE $g_{\text{Ens}}(\cdot; \{\rvtheta^{(j)}\}_{j=1}^J)$ trained on activations $\{\rva^{(n)}\}_{n=1}^N$, the expected reconstruction error can be decomposed into a bias term and a variance term. That is,
    \begin{align}
        &\E_{\{\rvtheta^{(j)}\}_{j=1}^J | \{\rva^{(n)}\}_{n=1}^N} \left[ \norm{2}{\rva^{(*)} -  g_{\text{Ens}}(\rva^{(*)}; \{\rvtheta^{(j)}\}_{j=1}^J)}^2 \right] \\
        =& \underbrace{\norm{2}{\rva^{(*)} - \E_{\{\rvtheta^{(j)}\}_{j=1}^J | \{\rva^{(n)}\}_{n=1}^N} [g_{\text{Ens}}(\rva^{(*)}; \{\rvtheta^{(j)}\}_{j=1}^J)]}^2 }_{\text{bias term}} \\
        +& \underbrace{\E_{\{\rvtheta^{(j)}\}_{j=1}^J | \{\rva^{(n)}\}_{n=1}^N} \left[ \norm{2}{\E_{\{\rvtheta^{(j)}\}_{j=1}^J | \{\rva^{(n)}\}_{n=1}^N}[g_{\text{Ens}}(\rva^{(*)}; \{\rvtheta^{(j)}\}_{j=1}^J)] -  g_{\text{Ens}}(\rva^{(*)}; \{\rvtheta^{(j)}\}_{j=1}^J)}^2 \right]}_{\text{variance term}}.
    \end{align}
\end{lemma}
\begin{proof}
    Since all the expectations are taken with respect to the same randomness, their subscripts are dropped for notational ease. Also, let $\Theta^{(J)} = \{\rvtheta^{(j)}\}_{j=1}^J$. We have
    \begin{align}
        &\E \left[
            \norm{2}{
                \rva^{(*)}
                - g_{\text{Ens}}(\rva^{(*)}; \Theta^{(J)})
            }^2
        \right] \\
        =& \E \left[
            \norm{2}{
                \rva^{(*)}
                - \E[g_{\text{Ens}}(\rva^{(*)}; \Theta^{(J)})]
                + \E[g_{\text{Ens}}(\rva^{(*)}; \Theta^{(J)})]
                - g_{\text{Ens}}(\rva^{(*)}; \Theta^{(J)})
            }^2
        \right] \\
        =&
        \E \left[
            \norm{2}{
                \rva^{(*)}
                - \E[g_{\text{Ens}}(\rva^{(*)}; \Theta^{(J)})]
            }^2
        \right] \label{eq:bias_in_expectation} \\
        &+
        \E \left[
            \norm{2}{
                \E[g_{\text{Ens}}(\rva^{(*)}; \Theta^{(J)})]
                - g_{\text{Ens}}(\rva^{(*)}; \Theta^{(J)})
            }^2
        \right] \\
        &+ 2\E \left[
            \left(
                \rva^{(*)} - \E[g_{\text{Ens}}(\rva^{(*)}; \Theta^{(J)})]
            \right)^\top
            \left(
                \E[g_{\text{Ens}}(\rva^{(*)}; \Theta^{(J)})]
                - g_{\text{Ens}}(\rva^{(*)}; \Theta^{(J)})
            \right)
        \right] \label{eq:decomp_last_term}.
    \end{align}
    Because $\rva^{(*)}$ and $\E[g_{\text{Ens}}(\rva^{(*)}; \Theta^{(J)})]$ are constants with respect to the expectation, for (\ref{eq:bias_in_expectation}) we have
    \begin{equation}
        \E \left[
            \norm{2}{
                \rva^{(*)}
                - \E[g_{\text{Ens}}(\rva^{(*)}; \Theta^{(J)})]
            }^2
        \right] = 
        \norm{2}{
            \rva^{(*)}
            - \E[g_{\text{Ens}}(\rva^{(*)}; \Theta^{(J)})]
        }^2,
    \end{equation}
    which is the stated bias term.

    For the last term in (\ref{eq:decomp_last_term}), we have
    \begin{align}
        &\E \left[
            \left(
                \rva^{(*)} - \E[g_{\text{Ens}}(\rva^{(*)}; \Theta^{(J)})]
            \right)^\top
            \left(
                \E[g_{\text{Ens}}(\rva^{(*)}; \Theta^{(J)})]
                - g_{\text{Ens}}(\rva^{(*)}; \Theta^{(J)})
            \right)
        \right] \\
        =& \left(
                \rva^{(*)} - \E[g_{\text{Ens}}(\rva^{(*)}; \Theta^{(J)})]
            \right)^\top
            \left(
                \E[g_{\text{Ens}}(\rva^{(*)}; \Theta^{(J)})]
                - \E[g_{\text{Ens}}(\rva^{(*)}; \Theta^{(J)})]
            \right) = 0,
    \end{align}
    again because $\rva^{(*)}$ and $\E[g_{\text{Ens}}(\rva^{(*)}; \Theta^{(J)})]$ are constants with respect to the expectation. Taken together, we have
    \begin{align}
        &\E \left[
            \norm{2}{
                \rva^{(*)}
                - g_{\text{Ens}}(\rva^{(*)}; \Theta^{(J)})
            }^2
        \right] \\
        =& \norm{2}{
            \rva^{(*)}
            - \E[g_{\text{Ens}}(\rva^{(*)}; \Theta^{(J)})]
        }^2
        + \E \left[
            \norm{2}{
                \E[g_{\text{Ens}}(\rva^{(*)}; \Theta^{(J)})]
                - g_{\text{Ens}}(\rva^{(*)}; \Theta^{(J)})
            }^2
        \right],
    \end{align}
    where the first term is the stated bias term, and the second term is the stated variance term.
\end{proof}

We now show that naive bagging (\Cref{sec:naive_bagging}) can reduce the reconstruction variance above. Formally, we have the following proposition.
\begin{proposition}\label{prop:naive_bagging_reconstruction}
    Given a neural network activation $\rva^{(*)}$ and the ensembled SAE $g_{\text{NB}}(\cdot; \{\rvtheta^{(j)}\}_{j=1}^J)$ obtained through naive bagging trained on activations $\{\rva^{(n)}\}_{n=1}^N$, the variance term in Lemma~\ref{lem:ensemble_bias_variance_decomp} goes to zero as $J \rightarrow \infty$.
\end{proposition}
\begin{proof}
    For notational ease, let $\rmA = \{\rva^{(n)}\}_{n=1}^N$, and $\Theta^{(J)} = \{\rvtheta^{(j)}\}_{j=1}^J$. By the definition of naive bagging, we have
    \begin{equation}
        g_{\text{NB}}(\rva^{(*)}; \Theta^{(J)}) = \frac{1}{J} \sum_{j=1}^J g(\rva^{(*)}; \theta^{(j)}).
    \end{equation}
    It follows that the variance term in Lemma~\ref{lem:ensemble_bias_variance_decomp} can be written as
    \begin{align}
        &\E_{\Theta^{(J)} | \rmA} \left[
            \norm{2}{
                \E_{\Theta^{(J)} | \rmA}\left[ \frac{1}{J} \sum_{j=1}^J g(\rva^{(*)}; \rvtheta^{(j)})\right]
                - \frac{1}{J} \sum_{j=1}^J g(\rva^{(*)}; \rvtheta^{(j)})
            }^2
        \right] \\
        =& \E_{\Theta^{(J)} | \rmA} \left[
            \norm{2}{
                \E_{\rvtheta | \rmA}[g(\rva^{(*)}; \rvtheta)] 
            - \frac{1}{J} \sum_{j=1}^J g(\rva^{(*)}; \rvtheta^{(j)})
            }^2
        \right], \label{eq:naive_bagging_identical_trainings}
    \end{align}
    where (\ref{eq:naive_bagging_identical_trainings}) follows from the linearity of expectation, and from the fact that $\rvtheta^{(1)}, ..., \rvtheta^{(J)}$ are identically distributed when conditioned on $\rmA$.
    
    For practical neural networks and SAEs, we can assume that
    \begin{equation}
        |g_q(\rva^{(*)}; \rvtheta)| < \infty
    \end{equation}
    for each dimension $q \in [d]$. Furthermore, conditioned on $\rmA$, the trainings of $\rvtheta^{(1)}, ..., \rvtheta^{(J)}$ are identically and independently distributed. Therefore, we have
    \begin{align}
        &\E_{\Theta^{(J)} | \rmA} \left[
            \norm{2}{
                \E_{\rvtheta | \rmA}[g(\rva^{(*)}; \rvtheta)] 
            - \frac{1}{J} \sum_{j=1}^J g(\rva^{(*)}; \rvtheta^{(j)})
            }^2
        \right] \\
        =& \E_{\Theta^{(J)} | \rmA} \left[
            \norm{2}{
                \frac{1}{J} \left(\sum_{j=1}^J \{g(\rva^{(*)}; \rvtheta^{(j)}) - \E_{\rvtheta | \rmA}[g(\rva^{(*)}; \rvtheta)]\}\right) 
            }^2
        \right] \\
        =& \frac{1}{J^2} \E_{\Theta^{(J)} | \rmA} \left[
            \sum_{j=1}^J \sum_{i=1}^J \left(g(\rva^{(*)}; \rvtheta^{(j)}) - \E_{\rvtheta | \rmA}[g(\rva^{(*)}; \rvtheta)]\right)^\top \left(g(\rva^{(*)}; \rvtheta^{(i)}) - \E_{\rvtheta | \rmA}[g(\rva^{(*)}; \rvtheta)]\right)
        \right] \\
        =& \frac{1}{J^2} \sum_{j=1}^J \E_{\theta^{(j)} | \rmA} \left[
            \norm{2}{
                g(\rva^{(*)}; \rvtheta^{(j)}) - \E_{\rvtheta | \rmA}[g(\rva^{(*)}; \rvtheta)] 
            }^2\right] \\
            &+ \frac{1}{J^2} \sum_{i \neq  j} \E_{\theta^{(j)}, \theta^{(i)} | \rmA}\left[  \left(g(\rva^{(*)}; \rvtheta^{(j)}) - \E_{\rvtheta | \rmA}[g(\rva^{(*)}; \rvtheta)]\right)^\top \left(g(\rva^{(*)}; \rvtheta^{(i)}) - \E_{\rvtheta | \rmA}[g(\rva^{(*)}; \rvtheta)]\right)
        \right] \\
        =& \frac{1}{J} \E_{\theta | \rmA} \left[
            \norm{2}{
                g(\rva^{(*)}; \rvtheta) - \E_{\rvtheta | \rmA}[g(\rva^{(*)}; \rvtheta)] 
            }^2\right] \\
            &+ \frac{1}{J^2} \sum_{i \neq  j} \E_{\rvtheta | \rmA}\left[  \left(g(\rva^{(*)}; \rvtheta) - \E_{\rvtheta | \rmA}[g(\rva^{(*)}; \rvtheta)]\right) \right]^\top \E_{\rvtheta | \rmA}\left[\left(g(\rva^{(*)}; \rvtheta) - \E_{\rvtheta | \rmA}[g(\rva^{(*)}; \rvtheta)]\right)
        \right] \\
        =& \frac{1}{J} \E_{\theta | \rmA} \left[
            \norm{2}{
                g(\rva^{(*)}; \rvtheta) - \E_{\rvtheta | \rmA}[g(\rva^{(*)}; \rvtheta)] 
            }^2\right],
    \end{align}
    where the third equality follows from the linearity of expectation, and the fourth equality from $\theta^{(1)}, ..., \theta^{(J)}$ having identical and independent distributions. The cross terms in the fourth equality sum to zero because of identical and independent distributions.
    With the assumption that $|g_q(\rva^{(*)}; \rvtheta)| < \infty$, we have
    \begin{equation}
            \norm{2}{
                g(\rva^{(*)}; \rvtheta) - \E_{\rvtheta | \rmA}[g(\rva^{(*)}; \rvtheta)] 
            }^2 < \infty.
    \end{equation}
    Therefore, the variance term in Lemma~\ref{lem:ensemble_bias_variance_decomp} goes to zero as $J \rightarrow \infty$.
\end{proof}

\begin{remark}
    We note that all the expectations in the bias-variance decomposition in Lemma~\ref{lem:ensemble_bias_variance_decomp} are conditioned on the specific training set $\{\rva^{(n)}\}_{n=1}^N$. This conditioning is needed for Proposition~\ref{prop:naive_bagging_reconstruction} to hold. Otherwise separate training runs of the SAE are dependent through the training set.
\end{remark}

We now discuss the two assumptions needed for bounding the bias term in Lemma~\ref{lem:ensemble_bias_variance_decomp} for boosting (\Cref{sec:boosting}).
\begin{assumption} \label{assume:boostin_generalization_bound}
    For a given neural network activation $\rva^{(*)}$ and the ensembled SAE $g_{\text{Boost}}(\cdot; \{\rvtheta^{(j)}\}_{j=1}^J)$ obtained through boosting trained on the activations $\{\rva^{(n)}\}_{n=1}^N$, we assume that
    \begin{align}
        &
        \norm{2}{
         \rva^{(*)} - \E_{\{\rvtheta^{(j)}\}_{j=1}^J | \{\rva^{(n)}\}_{n=1}^N} [g_{\text{Boost}}(\rva^{(*)}; \{\rvtheta^{(j)}\}_{j=1}^J)]   
        }^2 \\
        \le&
        \frac{1}{N} \sum_{n=1}^N \norm{2}{
            \rva^{(n)} - \E_{\{\rvtheta^{(j)}\}_{j=1}^J | \{\rva^{(n)}\}_{n=1}^N} [g_{\text{Boost}}(\rva^{(n)}; \{\rvtheta^{(j)}\}_{j=1}^J)]
        }^2 + \varepsilon_{G},
    \end{align}
    for some constant $\varepsilon_G > 0$.
\end{assumption}
\begin{remark}
    Assumption~\ref{assume:boostin_generalization_bound} is essentially a generalization bound on the reconstruction performance for boosting. Intuitively, this assumption can hold because SAEs are regularized. However, note that this assumption can break down when $\rva^{(*)}$ is much different from $\{\rva^{(n)}\}_{n=1}^N$, which is a general pitfall for generalization bounds.
\end{remark}

\begin{assumption} \label{assume:boosting_irreducible_error}
    For the ensembled SAE $g_{\text{Boost}}(\cdot; \{\rvtheta^{(j)}\}_{j=1}^J)$ obtained through boosting trained on the activations $\{\rva^{(n)}\}_{n=1}^N$, we assume that as $J \rightarrow \infty$,
    \begin{equation}
        \frac{1}{N} \sum_{n=1}^N \norm{2}{
            \rva^{(n)} - \E_{\{\rvtheta^{(j)}\}_{j=1}^J | \{\rva^{(n)}\}_{n=1}^N} [g_{\text{Boost}}(\rva^{(n)}; \{\rvtheta^{(j)}\}_{j=1}^J)]
        }^2 \le \varepsilon_{I},
    \end{equation}
    for some constant $\varepsilon_I > 0$.
\end{assumption}
\begin{remark}    Assumption~\ref{assume:boosting_irreducible_error} formalizes the intuition that boosting should be able to overfit almost perfectly to the training set. However, there is some irreducible error $\varepsilon_{I}$ because SAEs are simple and regularized models. This intuition is empirically verified in Supplementary Figure \ref{sfig:mse_loss_last_step}.
\end{remark}

We now present the proposition showing that boosting with more iterations can lead to a bounded bias term in Lemma~\ref{lem:ensemble_bias_variance_decomp}.
\begin{proposition}\label{prop:boosting_reconstruction}
    For a given neural network activation $\rva^{(*)}$ and the ensembled SAE $g_{\text{Boost}}(\cdot; \{\rvtheta^{(j)}\}_{j=1}^J)$ obtained through boosting trained on the activations $\{\rva^{(n)}\}_{n=1}^N$, under Assumption~\ref{assume:boostin_generalization_bound} and Assumption~\ref{assume:boosting_irreducible_error} we have, as $J \rightarrow \infty$,  
    \begin{equation}
            \norm{2}{
             \rva^{(*)} - \E_{\{\rvtheta^{(j)}\}_{j=1}^J | \{\rva^{(n)}\}_{n=1}^N} [g_{\text{Boost}}(\rva^{(*)}; \{\rvtheta^{(j)}\}_{j=1}^J)]   
            }^2 \le \varepsilon,
    \end{equation}
    for some constant $\varepsilon > 0$.
\end{proposition}
\begin{proof}
    The proof follows immediately under the assumptions. We have
    \begin{align}
        &
            \norm{2}{
             \rva^{(*)} - \E_{\{\rvtheta^{(j)}\}_{j=1}^J | \{\rva^{(n)}\}_{n=1}^N} [g_{\text{Boost}}(\rva^{(*)}; \{\rvtheta^{(j)}\}_{j=1}^J)]   
            }^2 \\
        \le& \frac{1}{N} \sum_{n=1}^N \norm{2}{
            \rva^{(n)} - \E_{\{\rvtheta^{(j)}\}_{j=1}^J | \{\rva^{(n)}\}_{n=1}^N} [g_{\text{Boost}}(\rva^{(n)}; \{\rvtheta^{(j)}\}_{j=1}^J)]
        }^2 + \varepsilon_{G} \label{eq:boosting_reconstruction_step1} \\
        \le& \varepsilon_I + \varepsilon_G \label{eq:boosting_reconstruction_step2},
    \end{align}
    where (\ref{eq:boosting_reconstruction_step1}) uses Assumption~\ref{assume:boostin_generalization_bound}, and (\ref{eq:boosting_reconstruction_step2}) uses Assumption~\ref{assume:boosting_irreducible_error}. Setting $\varepsilon = \varepsilon_I + \varepsilon_G$ completes the proof.
\end{proof}
\begin{remark}
    Proposition~\ref{prop:boosting_reconstruction} is not surprising given Assumption~\ref{assume:boostin_generalization_bound} and Assumption~\ref{assume:boosting_irreducible_error}. However, this formalization gives us insights about reasons why boosting may fail to reduce the bias term in the generalization region. That is, Assumption~\ref{assume:boostin_generalization_bound} or Assumption~\ref{assume:boosting_irreducible_error} may not hold (e.g. due to distribution shift or having too many constraints on the SAE, respectively).
\end{remark}
\begin{remark}
    Finally, we note that Proposition~\ref{prop:naive_bagging_reconstruction} and Proposition~\ref{prop:boosting_reconstruction} are both asymptotic results with respect to the number of SAEs in the ensemble, primarily serving to motivate naive bagging and boosting from the perspective of the reconstruction error. Future work that relates reconstruction to the identifiability of human-interpretable features would be more directly useful for downstream interpretability tasks.
\end{remark}

\begin{suppfigure}[h!]
\centering
\includegraphics[width=\textwidth]{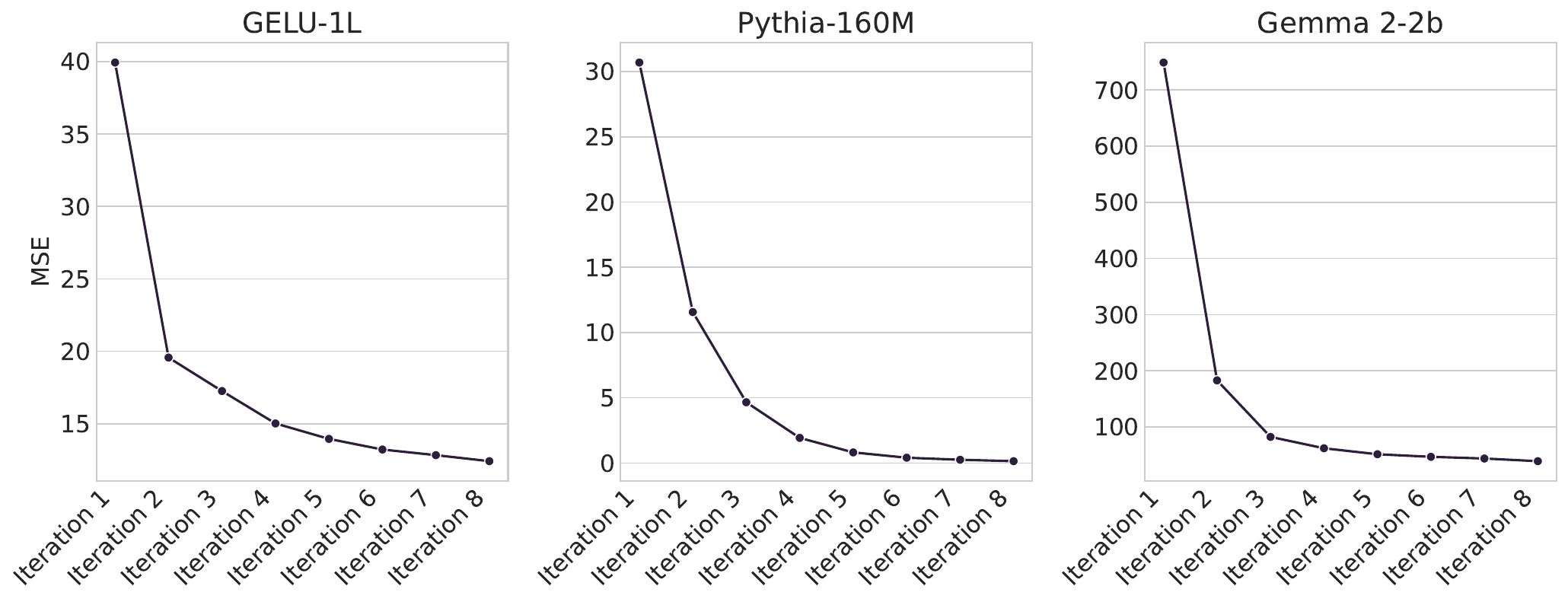}
\caption{MSE loss at the last training step for each iteration of a boosting ensemble with 8 SAEs. Reconstruction performance improves with each boosting iteration. 
} 
\label{sfig:mse_loss_last_step}
\end{suppfigure}

\clearpage
\section{Details on the Expanded SAE Baseline}
\label{app:expanded_sae}
For an expanded SAE, we choose the sparsity achieved by the corresponding boosted SAE as the target $L_0$. The sparsity of boosting instead of naive bagging is chosen for two reasons. Conceptually, naive bagging results in some redundant features, which contribute to higher total $L_0$ but do not reflect more diverse feature directions. Therefore, comparing with the $L_0$ of boosting provides a more representative baseline when assessing whether the expanded SAE can match or exceed the performance of an ensemble. Empirically, we observe that it is impractical to obtain an $L_0$ comparable to naive bagging even with a very small sparsity coefficient.

\section{Results for GELU-1L and Pythia-160M}
\label{app:add_intrinsic_eval}
Here we show the results for the intrinsic evaluations of GELU-1L (Supplementary Figure \ref{sfig:gelu_eval}) and Pythia-160M (Supplementary Figure \ref{sfig:pythia_eval}). Overall, the trend is similar to that of Gemma 2-2B; performance on most of the metrics improves as more SAEs are added to the ensemble, although it saturates for some of them around 8 SAEs. Also, boosting outperforms naive bagging in all metrics except for stability.

\begin{suppfigure}[h!]
\centering
\includegraphics[width=0.8\textwidth]{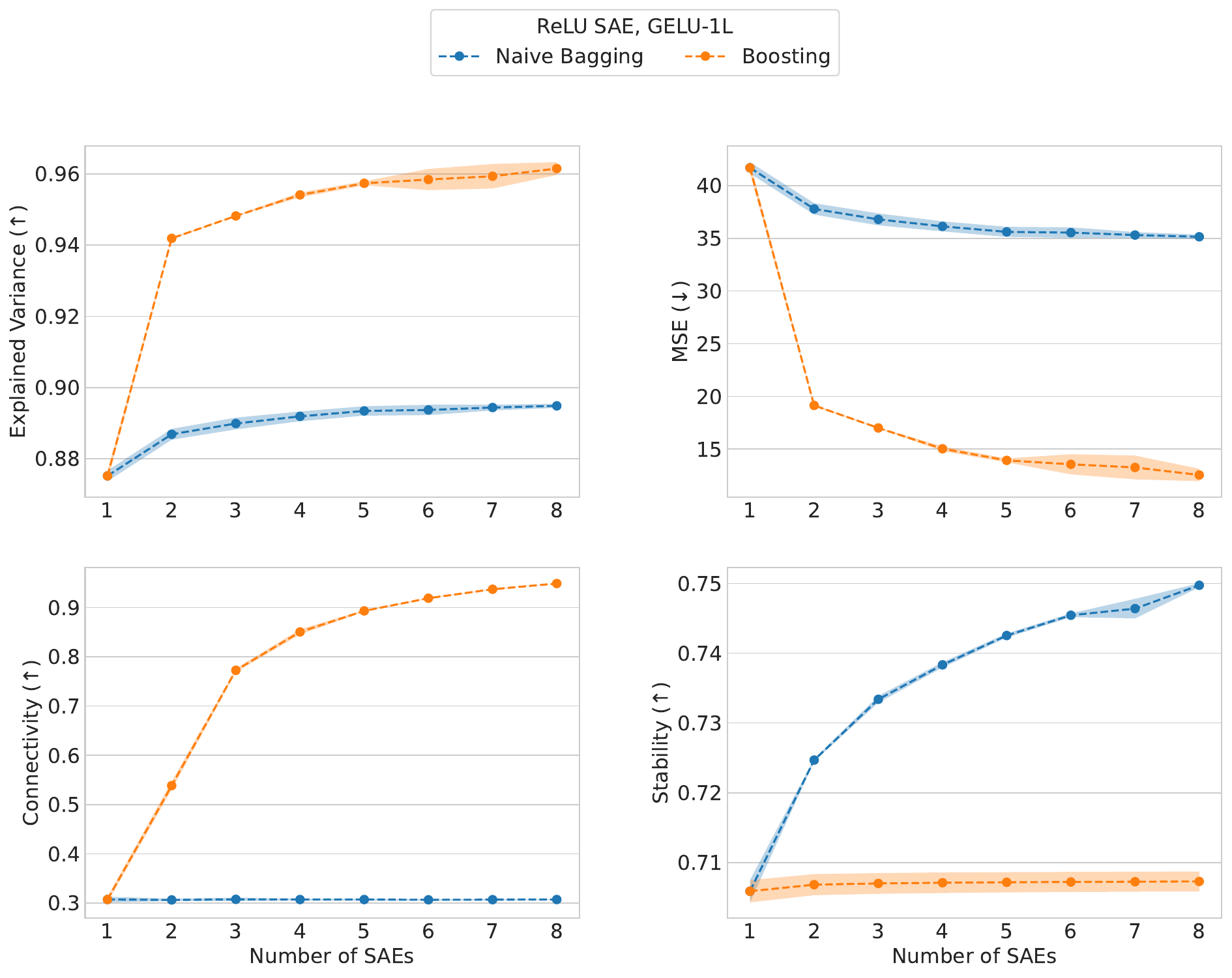}
\caption{Intrinsic evaluation of the ensembling approaches for GELU-1L. The shaded regions indicate 95\% confidence intervals across 5 experiment runs.
} 
\label{sfig:gelu_eval}
\end{suppfigure}

\begin{suppfigure}[h!]
\centering
\includegraphics[width=0.8\textwidth]{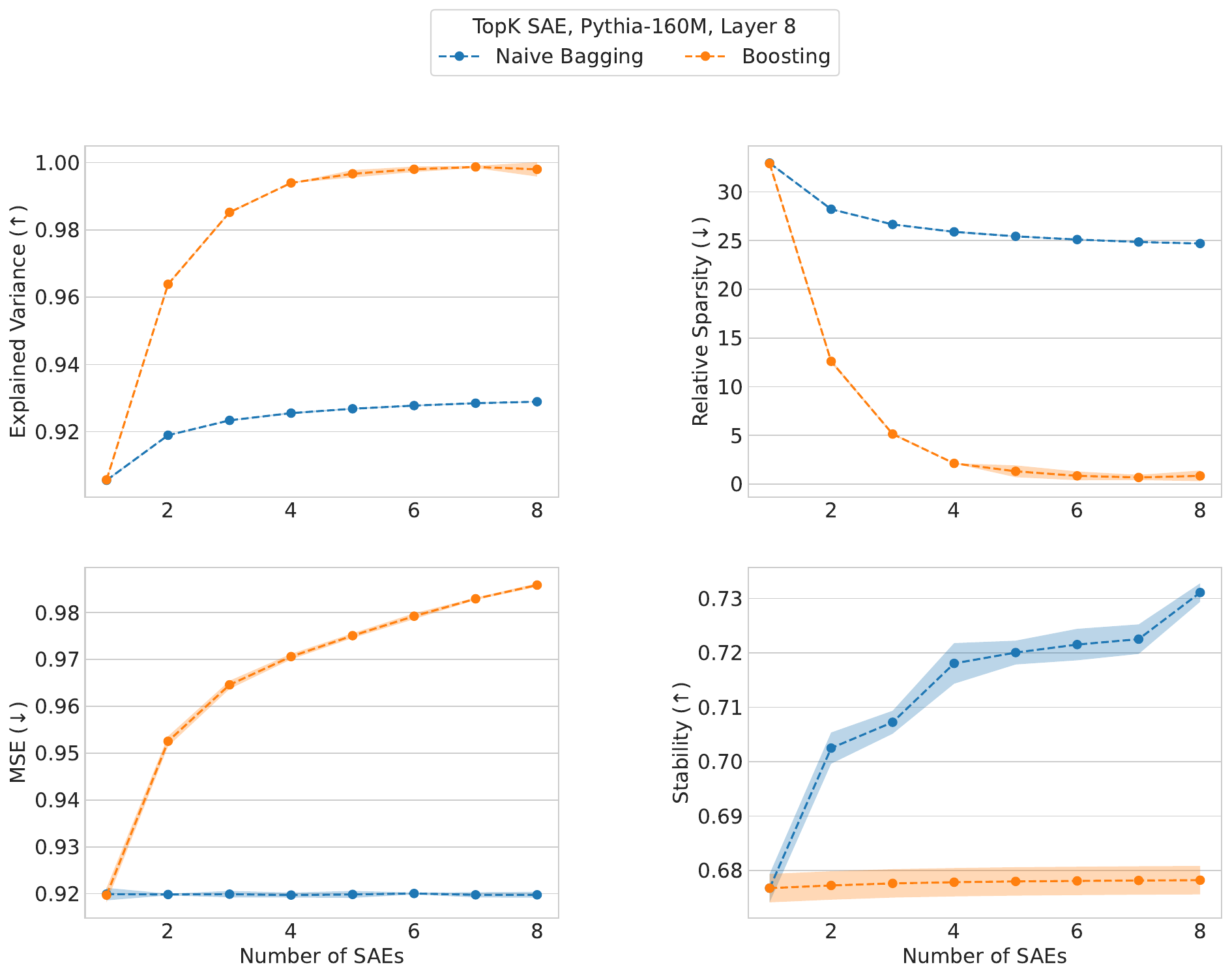}
\caption{Intrinsic evaluation of the ensembling approaches for Pythia-160M. The shaded regions indicate 95\% confidence intervals across 5 experiment runs.
} 
\label{sfig:pythia_eval}
\end{suppfigure}

\clearpage
\section{Sanity check with the Base SAE}
Here we provide a comparison of the ensembling methods with the base SAE on the intrinsic metrics (Supplementary Table \ref{stab:ensemble_eval}), the AutoInterp scores (Supplementary Table \ref{stab:autointerp}), the concept detection task (Supplementary Table \ref{stab:concept_detection}), and the SCR task (Supplementary Table \ref{stab:scr}). Across all the intrinsic metrics and the two use cases, we observe that ensembling can perform better than the base SAE.  
\setlength{\tabcolsep}{1pt}
\begin{supptable}[h!]
\centering
\small
\caption{Intrinsic evaluation metrics for the base SAE, naive bagging (NB), and boosting (ensembling 8 SAEs). Means along with 95\% confidence intervals are reported across 5 runs.}
\vspace{0.3em}
\resizebox{0.8\textwidth}{!}{
\begin{tabular}{p{3cm}p{2cm}p{2.3cm}p{2.2cm}l}

\toprule
\textbf{Ensembling Method} & \textbf{Explained \hspace{2em} Variance ($\uparrow$)}            & \textbf{MSE ($\downarrow$)}                & \textbf{Connectivity ($\uparrow$)}             & \textbf{Stability ($\uparrow$)}      \\
\midrule
\multicolumn{5}{l}{\textbf{GELU-1L}}                 \\
\midrule
Base SAE      & 0.875 (0.0020)                 & 41.694 (0.536)                    & 0.307 (0.0057)           & 0.705 (0.0016)          \\
Ensembling (NB)     & 0.895 (0.0006)                   & 35.147 (0.210)                         & 0.307 (0.0009)           & \textbf{0.745 (0.0002)} \\
Ensembling (Boosting)          & \textbf{0.961 (0.0018)}  & \textbf{12.542 (0.589)}    & \textbf{0.945 (0.0004)}  & 0.707 (0.0014)          \\
\midrule
\multicolumn{5}{l}{\textbf{Pythia-160M}}                                                                                                                                     \\
\midrule
Base SAE        & 0.906 (0.0003)                    & 32.965 (0.077)                       & 0.912 (0.0013)           & 0.677 (0.0026)          \\
Ensembling (NB)    & 0.929 (0.0000)           & 24.704 (0.019)             & 0.912 (0.0006)           & \textbf{0.731 (0.0017)} \\
Ensembling (Boosting)          & \textbf{0.998 (0.0021)} & \textbf{0.845 (0.547)}   & \textbf{0.986 (0.0004)}  & 0.680 (0.0025)          \\
\midrule
\multicolumn{5}{l}{\textbf{Gemma 2-2B}}                                                                                                                                      \\
\midrule
Base SAE      & 0.920 (0.0006)                  & 716.659 (5.875)           & 0.768 (0.0016)           & 0.581 (0.0006)          \\
Ensembling (NB)     & 0.974 (0.0006) & 234.128 (6.228) & 0.769 (0.0007)  & \textbf{0.633 (0.0014)} \\
Ensembling (Boosting)          & \textbf{0.995 (0.0003)} & \textbf{46.538 (2.923)} & \textbf{0.989 (0.0003)} & 0.583 (0.0009)   \\  
\bottomrule
\end{tabular}
}

\label{stab:ensemble_eval}
\end{supptable}

\begin{supptable}[h!]
\centering
\caption{AutoInterp scores for the base SAE, naive bagging (NB), and boosting (ensembling 8 SAEs). Means along with 95\% confidence intervals are reported across 5 experiment runs.}
\resizebox{0.6\textwidth}{!}{
\begin{tabular}{p{3.3cm}>{\centering}p{2cm}>{\centering}p{2cm}>{\centering\arraybackslash}p{2cm}>{\centering\arraybackslash}p{2cm}} \\
\toprule
\textbf{}      & \textbf{GELU-1L} & \textbf{Pythia-160M} & \textbf{Gemma 2-2B} \\
\midrule
Base SAE    & 0.690 (0.128)  & 0.840 (0.003) & 0.803 (0.006) \\
Ensembling (NB)                                    & 0.738 (0.147)                     & \textbf{0.857(0.004)}  &   0.799 (0.008)           \\
Ensembling (Boosting)        & \textbf{0.863 (0.005)}                & 0.852 (0.018) & \textbf{0.814 (0.002)} \\
\bottomrule
\end{tabular}
}
\label{stab:autointerp}
\end{supptable}

\begin{supptable}[h!]
\centering
\caption{Test accuracy of the logistic regression classifier for the top concept-associated feature across four concept detection tasks for Pythia-70M. SAE ensembles consist of 8 SAEs. Means along with 95\% confidence intervals are reported across 5 experiment runs.}
\vspace{0.3em}

\resizebox{0.7\textwidth}{!}{
\begin{tabular}{p{3.3cm}>{\centering}p{2cm}>{\centering}p{2cm}>{\centering\arraybackslash}p{2cm}>{\centering\arraybackslash}p{2cm}}
\toprule
\textbf{}      & \textbf{Amazon Review \hspace{2em} (Sentiment)} & \textbf{GitHub Code \hspace{4em} (Language)} & \textbf{AG News \hspace{2em}(Topic)} & \textbf{European Parliament \hspace{2em}(Language)} \\
\midrule
Base SAE   & 0.618 (0.030) & 0.711 (0.020)  & 0.733 (0.021) & 0.938 (0.016) \\
Ensembling (NB)                 & \textbf{0.631 (0.036)}                     & \textbf{0.715 (0.012)}                     & 0.742 (0.037)   &   \textbf{0.943 (0.016)}           \\
Ensembling (Boosting)        & 0.624 (0.037)            & 0.682 (0.021)            & \textbf{0.759 (0.021)} & 0.920 (0.015) \\
\bottomrule
\end{tabular}
}
\label{stab:concept_detection}
\end{supptable}

\begin{supptable}[h!]
\centering
\caption{$S_\text{SHIFT}$ scores for the spurious correlation removal task with the top 20 gender-related features identified across four pairs of professions for Pythia-70M. SAE ensembles consist of 8 SAEs. Means along with 95\% confidence intervals are reported across 5 experiment runs.}
\vspace{0.3em}
\resizebox{0.7\textwidth}{!}{
\begin{tabular}{l>{\centering}p{2cm}>{\centering}p{2cm}>{\centering}p{2.6cm}>{\centering\arraybackslash}p{2cm}}
\toprule
\textbf{}     & \textbf{Professor \hspace{2em} vs. Nurse} & \textbf{Architect \hspace{2em} vs. Journalist} & \textbf{Surgeon \hspace{4em} vs. Psychologist} & \textbf{Attorney \hspace{2em} vs. Teacher} \\
\midrule
Base SAE  & 0.039 (0.008) & 0.004 (0.006) & 0.027 (0.006)  & 0.017 (0.003) \\
Ensembling (NB) & 0.021 (0.003)                & 0.004 (0.001)                     & 0.014 (0.002)                     & 0.003 (0.005)                 \\
Ensembling (Boosting)      & \textbf{0.066 (0.016)}       & \textbf{0.013 (0.011)}            & \textbf{0.045 (0.014)}            & \textbf{0.029 (0.003)} \\
\bottomrule
\end{tabular}
}
\label{stab:scr}
\end{supptable}





\section{Additional Results for Downstream Use Cases}
\label{app:downstream}
Here we provide additional results for the downstream use cases (\Cref{sec:downstream_concept_detection} and \Cref{sec:downstream_scr}). 

Supplementary Table \ref{stab:concept_detection_top_5} shows the test accuracy of a classifier trained using the top-5 concept-associated features identified by the ensembling methods across four tasks. The results are slightly different from those in \Cref{sec:downstream_concept_detection}, with boosting outperforming naive bagging and the base SAE for three out of the four tasks. This suggests that, while boosting does not identify the top feature, additional features from boosting can be selected to improve concept detection.

\setlength{\tabcolsep}{5pt}
\begin{supptable}[h!]
\centering
\caption{Test accuracy of the logistic regression classifier for the top-5 concept-associated features across four concept detection tasks for Pythia-70M. SAE Ensembles consist of 8 SAEs. Means along with 95\% confidence intervals are reported across 5 experiment runs.}   
\begin{tabular}{l>{\centering}p{2cm}>{\centering}p{2cm}>{\centering}p{2cm}>{\centering\arraybackslash}p{2cm}>{\centering\arraybackslash}p{2cm}}
\toprule
\textbf{}    & \textbf{Amazon Review \hspace{2em} (Sentiment)} & \textbf{GitHub Code \hspace{4em} (Language)} & \textbf{AG News \hspace{2em}(Topic)} & \textbf{European Parliament \hspace{2em}(Language)} \\
\midrule
Base SAE   & 0.702 (0.015) & \textbf{0.805 (0.004)}  & 0.851 (0.005) & 0.981 (0.003) \\
Expanded SAE   & 0.703 (0.005) & 0.786 (0.012)  & 0.862 (0.011) & 0.986 (0.001) \\

Ensembling (NB) & 0.689 (0.015)                & 0.728 (0.005)                    & 0.783 (0.023)   &   0.952 (0.004)          \\
Ensembling (Boosting)   & \textbf{0.708 (0.016)}            & 0.795 (0.016)            & \textbf{0.863 (0.008)} & \textbf{0.988 (0.000)} \\
\bottomrule
\end{tabular}
\vspace{0.3em}

\label{stab:concept_detection_top_5}
\end{supptable}

Supplementary Figure \ref{sfig:scr_vs_topk_ablation} shows the $S_\text{SHIFT}$ scores for the spurious correlation removal task as the number of top gender-related features is varied. The trend is similar to what is observed in \Cref{sec:downstream_scr}, with boosting outperforming naive bagging and the baselines for different numbers of ablated gender-related SAE features. The performance generally increases as the number of ablated features increases, indicating that there are multiple gender related features which are correctly identified by all the methods. This is especially worth noting for naive bagging, as increasing the number of ablated features might lead to all the redundant features related to the spurious signal getting ablated.

\begin{suppfigure}[h!]
\centering
\includegraphics[width=\textwidth]{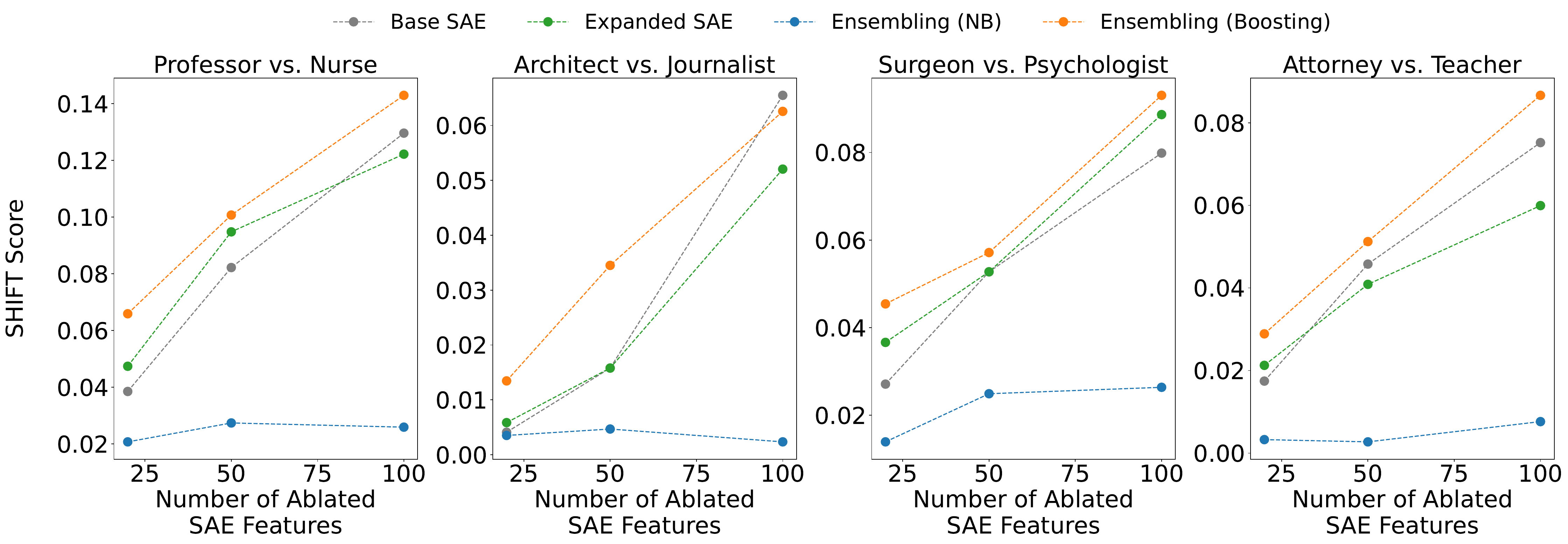}
\caption{$S_{\text{SHIFT}}$ scores for the spurious correlation removal task vs. various numbers of top gender-related features identified across four pairs of professions for Pythia-70M. SAE ensembles consist of 8 SAEs. Means across 5 experiment runs are shown.} 
\label{sfig:scr_vs_topk_ablation}
\end{suppfigure}



\rebuttal{We further evaluate the downstream use cases for the multi-layer language models used in the intrinsic evaluation, Pythia-160M and Gemma 2-2B.} \rebuttal{Supplementary Table \ref{stab:pythia_160m_concept_detection} shows the results of the concept detection task for the Pythia-160M model. Supplementary Table \ref{stab:gemma_2_2b_concept_detection_results} shows the results of the concept detection task for Gemma 2-2B. The results are consistent with those from the Pythia-70M model, with naive bagging outperforming the other methods on three out of the four datasets. For Pythia-160M, the expanded SAE performs better on the AG News (Topic) dataset while for Gemma 2-2B boosting performs better on the GitHub Code (Language) dataset.}

\rebuttal{Supplementary Table \ref{stab:pythia_160m_scr_results} shows the $S_\text{SHIFT}$ scores of the ensembling methods for the SCR task across four profession pairs for the Pythia-160M model. Supplementary Table \ref{stab:gemma_2_2b_scr_results} shows the $S_\text{SHIFT}$ scores for Gemma 2-2B. Overall, the results are consistent with those obtained for the Pythia-70M model, with the boosted ensemble outperforming the other methods for three out of the four profession pairs. For Pythia-160M, naive bagging performs better on the Attorney vs. Teacher pair while for Gemma 2-2B naive bagging performs better on the Architect vs. Journalist pair. It is worth noting that some of the scores for Gemma 2-2B are negative, which could happen when the probe accuracy after ablation is less than the baseline accuracy of the spurious probe before ablation.}

\setlength{\tabcolsep}{5pt}
\begin{supptable}[h!]
\centering
\caption{\rebuttal{Test accuracy of the logistic regression classifier for the top concept-associated feature across four concept detection tasks for Pythia-160M. SAE ensembles consist of 8 SAEs. Means along with 95\% confidence intervals are reported across 5 experiment runs.}}
\begin{tabular}{l>{\centering}p{2cm}>{\centering}p{2cm}>{\centering}p{2cm}>{\centering\arraybackslash}p{2cm}>{\centering\arraybackslash}p{2cm}}
\toprule
\textbf{}    & \textbf{\rebuttal{Amazon Review \hspace{2em} (Sentiment)}} & \textbf{\rebuttal{GitHub Code \hspace{4em} (Language)}} & \textbf{\rebuttal{AG News \hspace{2em}(Topic)}} & \textbf{\rebuttal{European Parliament \hspace{2em}(Language)}} \\
\midrule
\rebuttal{Base SAE} & \rebuttal{0.663 (0.009)} & \rebuttal{0.777 (0.032)} & \rebuttal{0.810 (0.009)} & \rebuttal{0.938 (0.012)} \\
\rebuttal{Expanded SAE} & \rebuttal{0.636 (0.007)} & \rebuttal{0.740 (0.059)} & \textbf{\rebuttal{0.826 (0.003)}} & \rebuttal{0.942 (0.005)} \\
\rebuttal{Ensembling (NB)} & \textbf{\rebuttal{0.702 (0.027)}} & \textbf{\rebuttal{0.813 (0.002)}} & \rebuttal{0.819 (0.012)} & \textbf{\rebuttal{0.945 (0.002)}} \\
\rebuttal{Ensembling (Boosting)} & \rebuttal{0.663 (0.013)} & \rebuttal{0.762 (0.029)} & \rebuttal{0.798 (0.032)} & \rebuttal{0.939 (0.008)} \\
\bottomrule
\end{tabular}
\label{stab:pythia_160m_concept_detection}
\end{supptable}


\begin{supptable}[h!]
\centering
\caption{\rebuttal{Test accuracy of the logistic regression classifier for the top concept-associated feature across four concept detection tasks for Gemma 2-2B. SAE ensembles consist of 8 SAEs. Means along with 95\% confidence intervals are reported across 5 experiment runs.}}
\begin{tabular}{l>{\centering}p{2cm}>{\centering}p{2cm}>{\centering}p{2cm}>{\centering\arraybackslash}p{2cm}>{\centering\arraybackslash}p{2cm}}
\toprule
\textbf{}    & \textbf{\rebuttal{Amazon Review \hspace{2em} (Sentiment)}} & \textbf{\rebuttal{GitHub Code \hspace{4em} (Language)}} & \textbf{\rebuttal{AG News \hspace{2em}(Topic)}} & \textbf{\rebuttal{European Parliament \hspace{2em}(Language)}} \\
\midrule
\rebuttal{Base SAE} & \rebuttal{0.797 (0.058)} & \rebuttal{0.666 (0.023)} & \rebuttal{0.828 (0.027)} & \rebuttal{0.835 (0.047)} \\
\rebuttal{Expanded SAE} & \rebuttal{0.786 (0.091)} & \rebuttal{0.676 (0.011)} & \rebuttal{0.811 (0.026)} & \rebuttal{0.738 (0.003)} \\
\rebuttal{Ensembling (NB)} & \textbf{\rebuttal{0.883 (0.029)}} & \rebuttal{0.691 (0.005)} & \textbf{\rebuttal{0.863 (0.009)}} & \textbf{\rebuttal{0.879 (0.008)}} \\
\rebuttal{Ensembling (Boosting)} & \rebuttal{0.816 (0.053)} & \textbf{\rebuttal{0.692 (0.004)}} & \rebuttal{0.828 (0.039)} & \rebuttal{0.867 (0.003)} \\
\bottomrule
\end{tabular}
\label{stab:gemma_2_2b_concept_detection_results}
\end{supptable}


\begin{supptable}[h!]
\centering
\caption{\rebuttal{$S_\text{SHIFT}$ scores for the spurious correlation removal task with the top 20 gender-related features identified across four pairs of professions for Pythia-160M. SAE Ensembles consist of 8 SAEs. Means along with 95\% confidence intervals are reported across 5 experiment runs.}}
\begin{tabular}{l>{\centering}p{2cm}>{\centering}p{2cm}>{\centering}p{2.6cm}>{\centering\arraybackslash}p{2cm}}
\toprule
\textbf{}     & \textbf{\rebuttal{Professor \hspace{2em} vs. Nurse}} & \textbf{\rebuttal{Architect \hspace{2em} vs. Journalist}} & \textbf{\rebuttal{Surgeon \hspace{4em} vs. Psychologist}} & \textbf{\rebuttal{Attorney \hspace{2em} vs. Teacher}} \\
\midrule
\rebuttal{Single SAE} & \rebuttal{0.626 (0.134)} & \rebuttal{0.485 (0.125)} & \rebuttal{0.543 (0.262)} & \rebuttal{0.318 (0.105)} \\
\rebuttal{Expanded SAE} & \rebuttal{0.559 (0.102)} & \rebuttal{0.703 (0.085)} & \rebuttal{0.729 (0.080)} & \rebuttal{0.358 (0.046)} \\
\rebuttal{Ensembling (NB)} & \rebuttal{0.700 (0.006)} & \rebuttal{0.770 (0.002)} & \rebuttal{0.709 (0.006)} & \textbf{\rebuttal{0.414 (0.002)}} \\
\rebuttal{Ensembling (Boosting)} & \textbf{\rebuttal{0.778 (0.038)}} & \textbf{\rebuttal{0.802 (0.053)}} & \textbf{\rebuttal{0.746 (0.108)}} & \rebuttal{0.341 (0.029)} \\
\bottomrule
\end{tabular}
\label{stab:pythia_160m_scr_results}
\end{supptable}


\begin{supptable}[h!]
\centering
\caption{\rebuttal{$S_\text{SHIFT}$ scores for the spurious correlation removal task with the top 20 gender-related features identified across four pairs of professions for Gemma 2-2B. SAE Ensembles consist of 8 SAEs. Means along with 95\% confidence intervals are reported across 5 experiment runs.}}
\begin{tabular}{l>{\centering}p{2cm}>{\centering}p{2cm}>{\centering}p{2.6cm}>{\centering\arraybackslash}p{2cm}}
\toprule
\textbf{}     & \textbf{\rebuttal{Professor \hspace{2em} vs. Nurse}} & \textbf{\rebuttal{Architect \hspace{2em} vs. Journalist}} & \textbf{\rebuttal{Surgeon \hspace{4em} vs. Psychologist}} & \textbf{\rebuttal{Attorney \hspace{2em} vs. Teacher}} \\
\midrule
\rebuttal{Base SAE} & \rebuttal{0.406 (0.007)} & \rebuttal{-0.391 (0.122)} & \rebuttal{0.206 (0.122)} & \rebuttal{-0.126 (0.141)} \\
\rebuttal{Expanded SAE} & \rebuttal{0.154 (0.127)} & \rebuttal{-0.339 (0.177)} & \rebuttal{0.291 (0.106)} & \rebuttal{-0.252 (0.269)} \\
\rebuttal{Ensembling (NB)} & \rebuttal{0.415 (0.137)} & \textbf{\rebuttal{-0.259 (0.029)}} & \rebuttal{0.121 (0.064)} & \rebuttal{-0.099 (0.061)} \\
\rebuttal{Ensembling (Boosting)} & \textbf{\rebuttal{0.594 (0.037)}} & \rebuttal{-0.528 (0.160)} & \textbf{\rebuttal{0.415 (0.177)}} & \textbf{\rebuttal{0.237 (0.076)}} \\
\bottomrule
\end{tabular}
\label{stab:gemma_2_2b_scr_results}
\end{supptable}


\section{Evaluation on other SAEBench Metrics}
\label{app:saebench}
\rebuttal{We have evaluated our ensembling approaches on two of the metrics implemented in SAEBench \cite{karvonen2025saebench}: concept detection and spurious correlation removal. Here we provide an extensive evaluation on the other metrics implemented in SAEBench for the multi-layer models Pythia-70M, Pythia-160M, and Gemma 2-2B. We don't include the unlearning metric since that metric is specifically for instruct models and none of our models are instruction tuned. For the feature absorption metric, the code repository suggests that this evaluation only makes sense if the LLM is large enough to have decent spelling knowledge and it is not recommended running this evaluation on LLMs with less than 1B parameters. Further, for the SAEs trained using Pythia-160M with our selected hyperparameters, the absorption could not be calculated due to insufficient first-letter features being detected. As a result, we run the feature absorption evaluation only on Gemma 2-2B, evaluating the mean full absorption rate.}

\rebuttal{Supplementary Table \ref{stab:sae_bench_pythia_70m} shows the results for SAEBench evaluation using SAEs trained on Pythia-70M. Ensembling approaches (specifically boosting) are able to perform slightly better than both the base SAE and the expanded SAE across all metrics in terms of the mean performance across five evaluation runs. The naive bagging approach doesn't perform well on the Targeted Probe Perturbation (TPP) task, which could be because this task is conceptually similar to Spurious Correlation Removal and the results are consistent with what was observed in that evaluation with naive bagging not performing well.}

\rebuttal{Supplementary Table \ref{stab:sae_bench_pythia_160m} shows the results for SAEBench evaluation using SAEs trained on Pythia-160M. Ensembling approaches (either naive bagging or boosting) perform better than the base SAE and the expanded SAE on two out of the three metrics in terms of the mean performance across five evaluation runs. The expanded SAE performs better on the TPP metric, while a similar performance drop is observed for naive bagging.}

\rebuttal{Supplementary Table \ref{stab:sae_bench_gemma_2_2b} shows the results for SAEBench evaluation using SAEs trained on Gemma 2-2B. Ensembling approaches (either naive bagging or boosting) perform slightly better than the base SAE and the expanded SAE on three out of the four metrics in terms of the mean performance across five evaluation runs. The expanded SAE performs better on the TPP metric, while a similar performance drop is observed for naive bagging.}

\rebuttal{Overall, our ensembling approaches are able to outperform or perform similarly to the baselines for most of the SAEBench metrics. This highlights that the performance improvements gained by ensembling are not limited to the concept detection and the SCR use cases, but can be generalized to multiple downstream tasks.}

\begin{supptable}[h!]
\centering
\caption{\rebuttal{Evaluation of the ensembling approaches on different SAEBench metrics for Pythia-70M. SAE Ensembles consist of 8 SAEs. Means along with 95\% confidence intervals are reported across 5 evaluation runs.}}
\resizebox{0.7\textwidth}{!}{
\begin{tabular}{lcccc}
\toprule
\rebuttal{\textbf{}} & \rebuttal{\textbf{RAVEL Score ($\uparrow$)}} & \rebuttal{\textbf{CE Loss Score ($\uparrow$)}} & \rebuttal{\textbf{TPP ($\uparrow$)}} \\
\midrule
\rebuttal{Base SAE} & \rebuttal{0.3078 (0.0003)} & \rebuttal{0.9822 (0.0012)} & \rebuttal{0.0164 (0.0015)} \\
\rebuttal{Expanded SAE} & \rebuttal{0.3077 (0.0001)}  & \rebuttal{0.9845 (0.0016)} & \rebuttal{0.0147 (0.0008)} \\
\rebuttal{Ensembling (NB)} & \rebuttal{0.3078 (0.0000)} & \rebuttal{0.9949 (0.0003)} & \rebuttal{0.0075 (0.0004)} \\
\rebuttal{Ensembling (Boosting)} & \textbf{\rebuttal{0.3080 (0.0001)}} & \textbf{\rebuttal{0.9974 (0.0002)}} & \textbf{\rebuttal{0.0187 (0.0005)}} \\
\bottomrule
\end{tabular}
}
\label{stab:sae_bench_pythia_70m}
\end{supptable}


\begin{supptable}[h!]
\centering
\caption{\rebuttal{Evaluation of the ensembling approaches on different SAEBench metrics for Pythia-160M. SAE Ensembles consist of 8 SAEs. Means along with 95\% confidence intervals are reported across 5 evaluation runs.}}
\resizebox{0.7\textwidth}{!}{
\begin{tabular}{lcccc}
\toprule
\rebuttal{\textbf{}} & \rebuttal{\textbf{RAVEL Score ($\uparrow$)}} & \rebuttal{\textbf{CE Loss Score ($\uparrow$)}} & \rebuttal{\textbf{TPP ($\uparrow$)}} \\
\midrule
\rebuttal{Base SAE} & \rebuttal{0.4969 (0.0010)} & \rebuttal{0.9771 (0.0004)} & \rebuttal{0.2635 (0.0149)} \\
\rebuttal{Expanded SAE} & \rebuttal{0.4970 (0.0010)} & \rebuttal{0.9783 (0.0001)} & \textbf{\rebuttal{0.2886 (0.0138)}} \\
\rebuttal{Ensembling (NB)} & \rebuttal{0.4977 (0.0008)} & \rebuttal{0.9829 (0.0000)} & \rebuttal{0.0163 (0.0012)} \\
\rebuttal{Ensembling (Boosting)} & \textbf{\rebuttal{0.5020 (0.0008)}} & \textbf{\rebuttal{1.0000 (0.0000)}} & \rebuttal{0.2701 (0.0128)} \\
\bottomrule
\end{tabular}
}
\label{stab:sae_bench_pythia_160m}
\end{supptable}


\setlength{\tabcolsep}{1pt}
\begin{supptable}[h!]
\centering
\caption{\rebuttal{Evaluation of the ensembling approaches on different SAEBench metrics for Gemma 2-2B. SAE Ensembles consist of 8 SAEs. Means along with 95\% confidence intervals are reported across 5 evaluation runs.}}
\resizebox{0.85\textwidth}{!}{
\begin{tabular}{l>{\centering}p{3cm}>{\centering}p{3cm}>{\centering}p{3cm}>{\centering\arraybackslash}p{4cm}>{\centering\arraybackslash}p{3.2cm}}
\toprule
\rebuttal{\textbf{}} &
\rebuttal{\textbf{RAVEL Score ($\uparrow$)}}  &
\rebuttal{\textbf{CE Loss Score ($\uparrow$)}} &
\rebuttal{\textbf{TPP ($\uparrow$)}} &
\rebuttal{\textbf{Mean Full \hspace{4em} Absorption Rate ($\downarrow$)}} \\
\midrule
\rebuttal{Base SAE} &
\rebuttal{0.7355 (0.0021)} &
\rebuttal{0.9780 (0.0004)} &
\rebuttal{0.0958 (0.0162)} &
\rebuttal{0.0060 (0.0001)} \\
\rebuttal{Expanded SAE} &
\rebuttal{0.7468 (0.0042)} &
\rebuttal{0.9780 (0.0002)} &
\rebuttal{\textbf{0.1377 (0.0102)}} &
\rebuttal{0.0023 (0.0007)} \\
\rebuttal{Ensembling (NB)} &
\rebuttal{\textbf{0.7625 (0.0010)}} &
\rebuttal{0.9997 (0.0001)} &
\rebuttal{0.0162 (0.0008)} &
\rebuttal{0.0039 (0.0002)} \\
\rebuttal{Ensembling (Boosting)} &
\rebuttal{0.7520 (0.0002)} &
\rebuttal{\textbf{1.0000 (0.0000)}} &
\rebuttal{0.1168 (0.0070)} &
\rebuttal{\textbf{0.0020 (0.0010)}} \\
\bottomrule
\end{tabular}
}
\label{stab:sae_bench_gemma_2_2b}
\end{supptable}

\section{Implementation Details}
\label{app:implementation_details}
Here we provide additional details about the data, compute, and hyperparameter selection.

\subsection{Dataset and Models}

The Pile dataset~\citep{gao2020pile} (with copyrighted contents removed) used for training the SAEs is a large, diverse, and open-source English text dataset curated specifically for training general-purpose language models. Its diverse components include academic papers (e.g., arXiv, PubMed Central), books (e.g., Books3, BookCorpus2), code (from GitHub), web content (e.g., a filtered version of Common Crawl called Pile-CC, OpenWebText2), and other sources like Wikipedia, Stack Exchange, and subtitles. Beyond training, the Pile also serves as a benchmark for evaluating language models. More recently, the Pile has become the standard dataset for training sparse autoencoders~\citep{bussmann2025learning,cunningham2023sparse,lieberum2024gemma,marks2024sparse,paulo2025sparse}.

All the language models we use have been previously used for training and evaluating sparse autoencoders~\citep{bricken2023monosemanticity, gao2024scaling, lieberum2024gemma,paulo2025sparse}. Supplementary Table \ref{stab:lm_overview} provides additional details on the language models and the corresponding SAE architectures.

\setlength{\tabcolsep}{4.5pt}
\begin{supptable}[h!]
\centering
\caption{Overview of the language models and SAE architectures used for intrinsic evaluation and downstream use cases.}
\begin{tabular}{p{2.6cm}>{\centering}p{1.2cm}>{\centering}p{1.2cm}>{\centering}p{1cm}>{\centering}p{2cm}>{\centering}p{1cm}c}
\toprule
\textbf{Language Model} & \textbf{Num. Params} & \textbf{Num. Layers} & \textbf{Context Size} & \textbf{Activation Dimension} & \textbf{Layer Used} & \textbf{SAE Arch.} \\
\midrule
\multicolumn{3}{l}{\textbf{Intrinsic Evaluation}}    & & & & \\
\midrule
GELU-1L & 3.1M   & 1 & 1024   & 512  & 1   & ReLU    \\
Pythia-160M  & 162.3M   & 12  & 2048 & 768  & 8   & TopK  \\
Gemma 2-2B  & 2.1B    & 26  & 8192  & 2304   & 12 & JumpReLU  \\
\midrule
\multicolumn{3}{l}{\textbf{Downstream Use Cases}}  & & & & \\
\midrule
Pythia-70M & 70.4M   & 6  & 2048 & 512  & 4   & ReLU  \\
\bottomrule

\end{tabular}
\label{stab:lm_overview}
\end{supptable}

\subsection{Training}
\label{app:training}
Our ensembling algorithms are implemented in PyTorch\footnote{\url{https://pytorch.org/}} by adapting the SAELens library.\footnote{\url{https://github.com/jbloomAus/SAELens/tree/main}} The pseudocode for boosting is summarized in Algorithm \ref{alg:boosting_training}. For naive bagging, the training procedure for each SAE in the ensemble is the same as the standard SAE training. All the SAEs and the ensembles are trained on either an A100 GPU with 80GB of memory or an H100 NVL GPU with 93 GB of memory using a batch size of 10000. Supplementary Table~\ref{stab:training_times} shows the time taken for a single experiment run on a single H100 GPU for ensembles with 8 SAEs. It is worth noting that naive bagging can be parallelized across multiple GPUs, bringing down the training time to that of the base SAE when the number of GPUs is equal to the number of SAEs in the ensemble. \rebuttal{Supplementary Table~\ref{stab:infernce_times} shows the inference times for the different SAEs, which is calculated by doing a forward pass through the trained SAE for 10 batches. The inference time for the ensembling approaches is higher than the baselines, but they are essentially the same (around 1 second) for all practical purposes.}

\begin{supptable}[h!]
\centering
\caption{Training times for the base SAE, expanded SAE, and one experiment run for ensembles with 8 SAEs on a single H100 GPU.}
\vspace{0.3em}
\begin{tabular}{lllll}
\toprule
\textbf{}     & \textbf{GELU-1L} & \textbf{Pythia-160M} & \textbf{Gemma 2-2B} & \textbf{Pythia-70M} \\
\midrule
Base SAE  & 3h 2m & 5h 43m & 11h 7m & 21m \\
\rebuttal{Expanded SAE} & \rebuttal{19h 2m} & \rebuttal{1d 5h 29m} & \rebuttal{2d 11h 41m} & \rebuttal{3h 48m} \\
Ensembling (NB) & 1d 0h 16m   & 1d 21h 44m & 3d 16h 56m  & 3h 56m\\
Ensembling (Boosting)      & 1d 8h 26m  & 2d 0h 17m &  5d 5h 26m & 5h 35m \\
\bottomrule
\end{tabular}
\label{stab:training_times}
\end{supptable}

\begin{supptable}[h!]
\centering
\caption{\rebuttal{Inference times (in seconds) for the base SAE, expanded SAE, and ensembling approaches with 8 SAEs on a single A100 GPU. The mean time along with 95\% confidence intervals of a forward pass through the SAEs across 10 batches are reported.}}
\begin{tabular}{lllll}
\toprule
 & \textbf{\rebuttal{GELU-1L}} & \textbf{\rebuttal{Pythia-160M}} & \textbf{\rebuttal{Gemma 2-2B}} & \textbf{\rebuttal{Pythia-70M}} \\
\midrule
\rebuttal{Base SAE} & \rebuttal{0.156 (0.209)} & \rebuttal{0.178 (0.312)} & \rebuttal{0.302 (0.503)} & \rebuttal{0.105 (0.053)} \\
\rebuttal{Expanded SAE} & \rebuttal{0.165 (0.223)} & \rebuttal{0.184 (0.322)} & \rebuttal{0.338 (0.551)} & \rebuttal{0.170 (0.191)} \\
\rebuttal{Ensembling (NB)} & \rebuttal{0.208 (0.212)} & \rebuttal{0.549 (0.301)} & \rebuttal{1.400 (0.515)} & \rebuttal{0.223 (0.228)} \\
\rebuttal{Ensembling (Boosting)} & \rebuttal{0.244 (0.422)} & \rebuttal{0.889 (1.483)} & \rebuttal{1.630 (0.910)} & \rebuttal{0.249 (0.240)} \\
\bottomrule
\end{tabular}
\label{stab:infernce_times}
\end{supptable}

\begin{algorithm}[H]
\label{alg:boosting}
\SetAlgoLined
\DontPrintSemicolon
\KwInput{Training activations $\{\rva^{(n)}\}_{n=1}^N$, learning rate $\alpha$, sparsity coefficient $\lambda$, sparsity norm coefficient $p$, activation function $h(\cdot)$, previous SAEs $[g(\cdot; \rvtheta^{(1)}),...,g(\cdot; \rvtheta^{(j-1)})]$}
\KwOutput{Trained SAE $g(\cdot; \rvtheta^{(j)})$}
\BlankLine
\tcp{Randomly initialize weights}
initialize parameters $\rvtheta^{(j)}$ (i.e. $\Wenc^{(j)}, \benc^{(j)}, \Wdec^{(j)}, \bdec^{(j)}$) \;
initialize $n=0$
\BlankLine
\While{$n<N$}{
\tcp{Determine residual from previous SAEs}
initialize $\ve=$ zeros\_like($\rva^{(n)}$)
\BlankLine
\For{$\ell \in [j-1]$}{
update $\ve \leftarrow \ve + g(\rva^{(n)}-\ve; \rvtheta^{(\ell)})$
}
\tcp{Leftover residual}
set $\vr=\rva^{(n)} - \ve$ \;
\tcp{Determine predicted residual and feature coefficients}
calculate $\hat{\vr}=g(\vr; \rvtheta^{(j)})$ \;
calculate $\vc = h(\Wenc^{(j)} \vr + \benc^{(j)})$ \;
\tcp{Calculate loss}
set $\cL_{\text{Boost}}\left(\rva^{(n)}; \rvtheta^{(j)} \right) = \norm{2}{\rvr - \hat{\rvr}}^2 + \lambda 
 \norm{p}{\vc}$ \;
 \tcp{Gradient step}
 update $\rvtheta^{(j)}\leftarrow \theta^{(j)}-\alpha\nabla_{\rvtheta^{(j)}}\cL_{\text{Boost}}\left(\rva^{(n)}; \rvtheta^{(j)} \right)$\;
 \tcp{update $n$}
 update $n\leftarrow n+1$
}
\caption{Training algorithm for the $j$th iteration of boosting. Gradient descent with a mini-batch size of 1 is shown as an illustration.}
\label{alg:boosting_training}
\end{algorithm}

\subsection{Hyperparameter Selection}
For the smallest model (GELU-1L), we conduct an extensive hyperparameter search across the learning rate, sparsity coefficient, and the expansion factor (Supplementary Figure \ref{sfig:gelu_hyperparam}), where the expansion factor refers to the multiplicative factor for the input activation dimensionality to get the SAE's hidden dimensionality ($k = d \times \text{Expansion Factor}$). We select the hyperparameters that get closest to 90\% explained variance while having the smallest L0 to ensure that the reconstructions are faithful to the original activations and the SAE decompositions are sparse.

For the larger Pythia-160M and Gemma 2-2B, we use the same learning rate from GELU-1L and consider expansion factors which give SAEs with a similar dimensionality ($k$) as the SAE for GELU-1L. We perform a sweep over the hyperparameter which controls the sparsity of the SAE (TopK value for Pythia-160M and the L0 coefficient for Gemma 2-2B) and select the values that give us an explained variance closest to 90\% (Supplementary Figure \ref{sfig:pythia_gemma_hyperparam}). For the Pythia-70M model used in the downstream tasks, the hyperparameter values were borrowed from prior work~\citep{karvonen2024evaluating, marks2024sparse}. The final selected hyperparameter values are provided in Supplementary Table \ref{stab:hyperparam_selections}.

\begin{suppfigure}[h!]
\centering
\includegraphics[width=\textwidth]{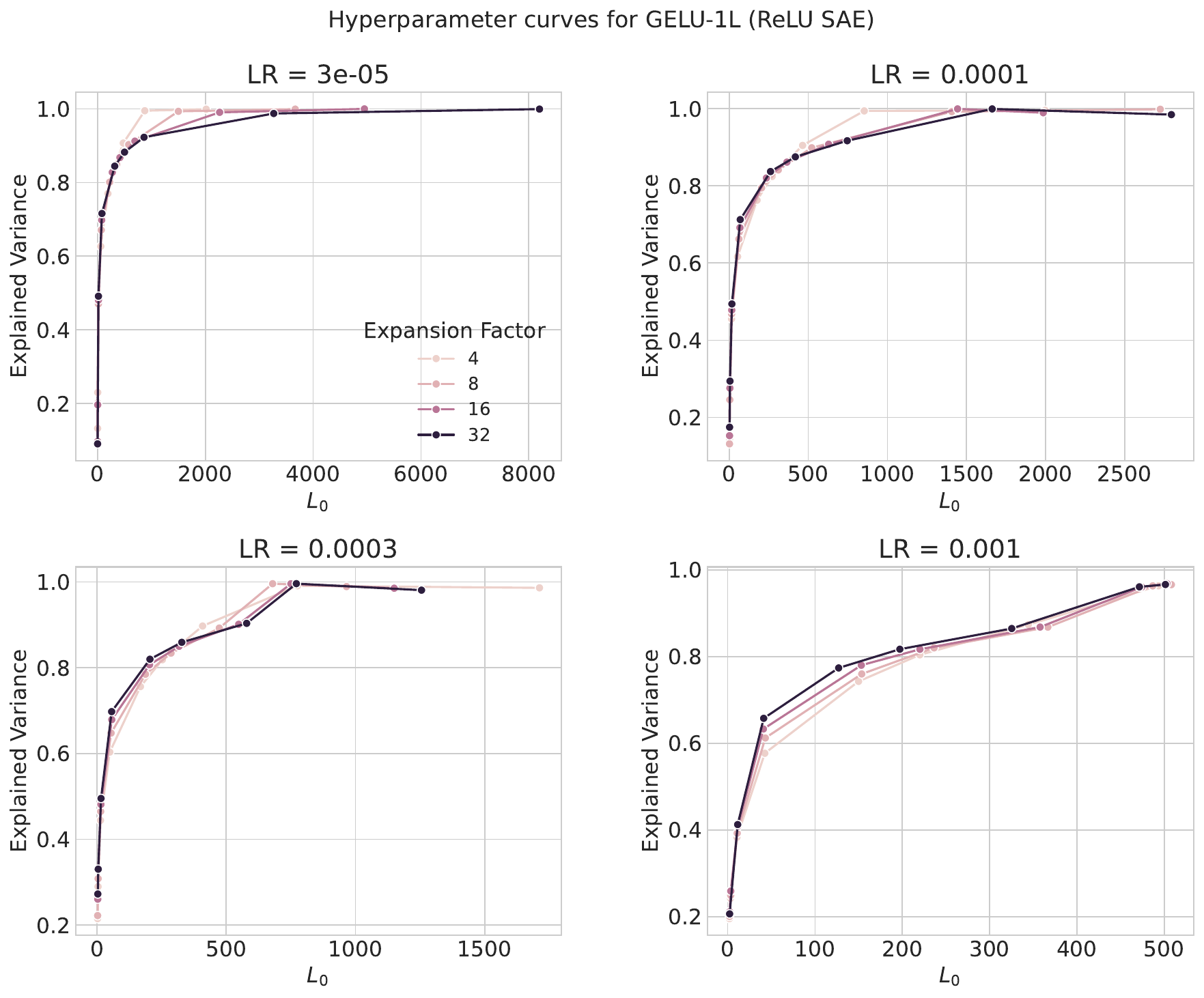}
\caption{Hyperparameter sweep performed for the GELU-1L activations with the ReLU SAE across different learning rates, expansion factors, and sparsity coefficients.
} 
\label{sfig:gelu_hyperparam}
\end{suppfigure}

\begin{suppfigure}[t!]
\centering
\includegraphics[width=\textwidth]{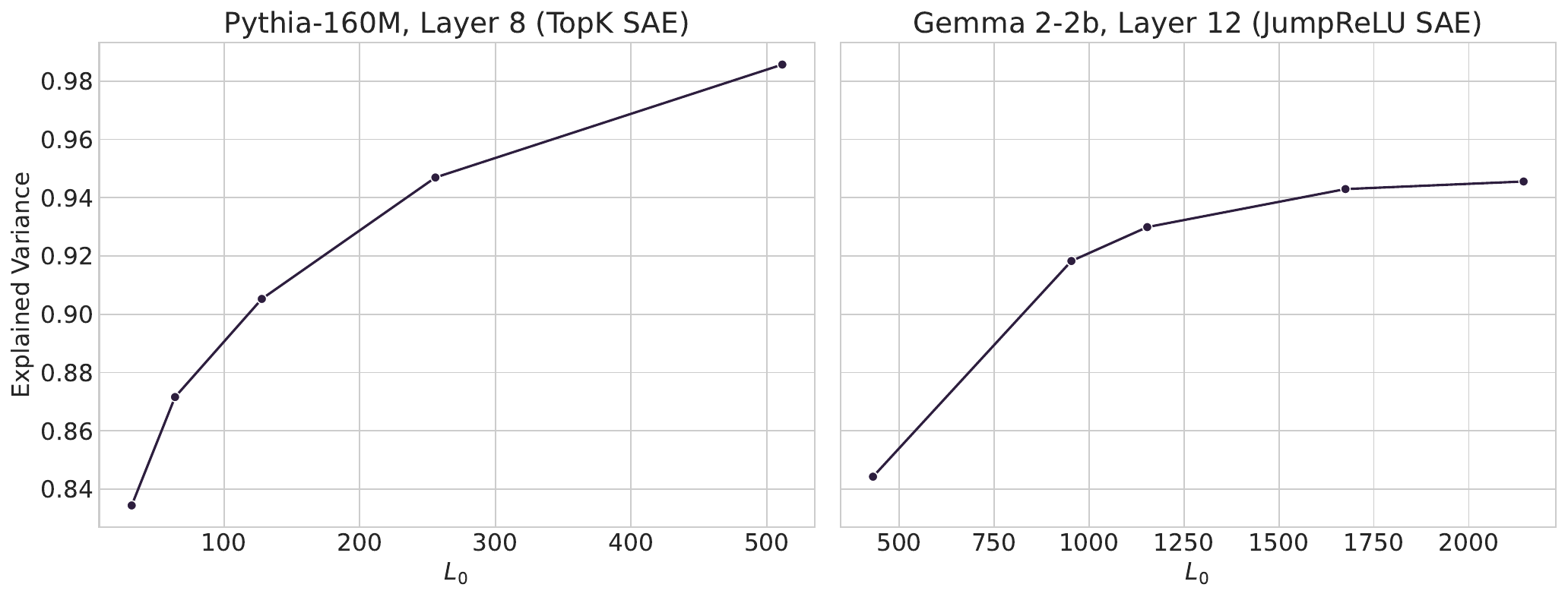}
\caption{Hyperparameter sweep performed for the Pythia-160M and Gemma 2-2B activations with the TopK and JumpReLU SAEs, respectively. For Pythia-160M, the sweep is across different values of $K$, 
 and for Gemma 2-2B it is across different sparsity coefficients.
} 
\label{sfig:pythia_gemma_hyperparam}
\end{suppfigure}

\setlength{\tabcolsep}{4.5pt}
\begin{supptable}[t]
\centering
\caption{Selected hyperparameter values for the base SAE. These hyperparameters are held constant for all SAEs in the ensemble.}
\begin{tabular}{lcccc}
\toprule
\textbf{Language Model} & \textbf{Learning Rate} & \textbf{Expansion Factor} & \textbf{TopK} & \textbf{Sparsity Coefficient} \\
\midrule
\multicolumn{2}{l}{\textbf{Intrinsic Evaluation}}    & & &   \\
\midrule
GELU-1L & 0.0003   & 32 & --   & 0.75     \\
Pythia-160M  & 0.0003   & 21  & 128 & --  \\
Gemma 2-2B  & 0.0003   & 7  & --  & 0.75  \\
\midrule
\multicolumn{2}{l}{\textbf{Downstream Use Cases}}    & & &   \\
\midrule
Pythia-70M & 0.0001 & 64 & -- & 0.1 \\
\bottomrule

\end{tabular}
\label{stab:hyperparam_selections}
\end{supptable}

\end{document}